\def\year{2019}\relax
\documentclass[letterpaper]{article} 
\usepackage{aaai19}  
\usepackage{times}  
\usepackage{helvet}  
\usepackage{courier}  
\usepackage{url}  
\usepackage{graphicx}  
\usepackage{microtype}
\usepackage{subfig}
\usepackage{graphicx}
\usepackage{booktabs} 
\usepackage{hyperref}
\usepackage{amsthm}
\usepackage{amsmath}
\usepackage{amssymb}
\usepackage{comment}
\usepackage{algorithm}
\usepackage{algorithmic}

\usepackage{pgf}
\usepackage{tikz}
\usetikzlibrary{arrows,automata,positioning}

\makeatletter
\newcommand*\bigcdot{\mathpalette\bigcdot@{.5}}
\newcommand*\bigcdot@[2]{\mathbin{\vcenter{\hbox{\scalebox{#2}{$\m@th#1\bullet$}}}}}
\makeatother

\usepackage{multicol}

\newtheorem{proposition}{Proposition}

\newtheorem{lemma}[proposition]{Lemma}
\newtheorem{theorem}[proposition]{Theorem}

\newtheorem{defn}{Definition}
\newtheorem{remark}{Remark}

\def\eqdef{\stackrel{\text{def}}{=}}

\def\E{{\mathbb E}}

\def\G{{\mathcal G}}
\def\S{{\mathcal S}}

\newcommand{\condE}[1]{\E_{\mid #1}}

\usepackage{color}

\newcommand{\approxhgreedy}[2]{\G^{#1}_{h}(#2)}

\begin{document}
%
\title{How to Combine Tree-Search Methods in Reinforcement Learning}
\author{Yonathan Efroni \\
Technion, Israel\\
\And
Gal Dalal \\
Technion, Israel\\
\And
Bruno Scherrer  \\
INRIA, Villers les Nancy, France\\
\And
  Shie Mannor \\
 Technion, Israel\\
}
\maketitle
\begin{abstract}

Finite-horizon lookahead policies are abundantly used in Reinforcement Learning and demonstrate impressive empirical success. Usually, the lookahead policies are implemented with specific planning methods such as Monte Carlo Tree Search (e.g. in AlphaZero \cite{silver2017mastering}). Referring to the planning problem as tree search, a reasonable practice in these implementations is to back up the value only at the leaves while the information obtained at the root is not leveraged other than for updating the policy. Here, we question the potency of this approach.
Namely, the latter procedure is non-contractive in general, and its convergence is not guaranteed. Our proposed enhancement is straightforward and simple: use the return from the optimal tree path to back up the values at the descendants of the root. This leads to a $\gamma^h$-contracting procedure, where $\gamma$ is the discount factor and $h$ is the tree depth. To establish our results, we first introduce a notion called \emph{multiple-step greedy consistency}. We then provide convergence rates for two algorithmic instantiations of the above enhancement in the presence of noise injected to both the tree search stage and value estimation stage.

%
\end{abstract}
\section{Introduction}
A significant portion of the Reinforcement Learning (RL) literature regards Policy Iteration (PI) methods. This family of algorithms contains numerous variants which were thoroughly analyzed \cite{puterman1994markov,bertsekas1995neuro} and constitute the foundation of sophisticated state-of-the-art implementations \cite{mnih2016asynchronous,silver2017mastering}. The principal mechanism of PI is to alternate between policy evaluation and policy improvement. Various well-studied approaches exist for the policy evaluation stages; these may rely on single-step bootstrap, multi-step Monte-Carlo return, or parameter-controlled interpolation of the former two. For the policy improvement stage, theoretical analysis was mostly reserved for policies that are 1-step greedy, while recent prominent implementations of multiple-step greedy policies exhibited promising empirical behavior \cite{silver2017mastering,silver2017mastering2}. 

Relying on recent advances in the analysis of multiple-step lookahead policies \cite{beyond2018efroni,efroni2018multiple}, we study the convergence of a PI scheme whose improvement stage is $h$-step greedy with respect to (w.r.t.) the value function, for $h>1.$ Calculating such policies can be done via Dynamic Programming (DP) or other planning methods such as tree search. Combined with sampling, the latter corresponds to the famous Monte Carlo Tree Search (MCTS) algorithm employed in \cite{silver2017mastering,silver2017mastering2}. In this work, we show that even when partial (inexact) policy evaluation is performed and noise is added to it, along with a noisy policy improvement stage, the above PI scheme converges with a $\gamma^h$ contraction coefficient. While doing so, we also isolate a sufficient convergence condition which we refer to as \emph{$h$-greedy consistency} and relate it to previous 1-step greedy relevant literature. 

A straightforward `naive' implementation of the PI scheme described above would perform an $h$-step greedy policy improvement and then evaluate that policy by bootstrapping the `usual' value function. 
Surprisingly, we find that this procedure does not necessarily contracts toward the optimal value, and give an example where it is indeed non-contractive. This contraction coefficient depends both on $h$ and on the partial evaluation parameter: $m$ in the case of $m$-step return, and $\lambda$ when eligibility trace is used. The non-contraction occurs even when the $h$-greedy consistency condition is satisfied. 

To solve this issue, we propose an easy fix which we employ in all our algorithms, that relieves the convergence rate from the dependence of $m$ and $\lambda$, and  allows the $\gamma^h$ contraction mentioned earlier in this section. Let us treat each state as a root of a tree of depth $h;$ then our proposed fix is the following. Instead of backing up the value only at the leaves and ridding of all non-root related tree-search outputs, we reuse the tree-search byproducts and back up the optimal value of the root node children. Hence, instead of bootstrapping the `usual' value function in the evaluation stage, we bootstrap the optimal value obtained from the $h-1$ horizon optimal planning problem.

The contribution of this work is primarily theoretical, but in Section~\ref{sec: experiments} we also present experimental results on a toy domain. The experiments support our analysis by exhibiting better performance of our enhancement above compared to the `naive' algorithm. Additionally, we identified previous practical usages of this enhancement in literature. In \cite{baxter1999tdleaf}, the authors proposed backing up the optimal tree search value as a heuristic. They named the algorithm TDLeaf($\lambda$) and showcase its outperformance over the alternative `naive' approach. A more recent work \cite{lai2015giraffe} introduced a deep learning implementation of TDLeaf($\lambda$) called Giraffe. Testing it on the game of Chess, the authors claim (during publication) it is ``the most successful attempt thus far at using end-to-end machine learning to play chess''. In light of our theoretical results and empirical success described above, we argue that backing up the optimal value from a tree search should be considered as a `best practice' among RL practitioners.

\section{Preliminaries}
Our framework is the infinite-horizon discounted Markov Decision Process (MDP). An MDP is defined as the 5-tuple $(\mathcal{S}, \mathcal{A},P,R,\gamma)$ \cite{puterman1994markov}, where ${\mathcal S}$ is a finite state space, ${\mathcal A}$ is a finite action space, $P \equiv P(s'|s,a)$ is a transition kernel, $R \equiv r(s,a)\in[R_{\min},R_{\max}]$ is a reward function, and $\gamma\in(0,1)$ is a discount factor. Let $\pi: \mathcal{S}\rightarrow \mathcal{P}(\mathcal{A})$ be a stationary policy, where $\mathcal{P}(\mathcal{A})$ is a probability distribution on $\mathcal{A}$. Let $v^\pi \in \mathbb{R}^{|\mathcal{S}|}$ be the value of a policy $\pi,$ defined in state $s$ as $v^\pi(s) \equiv \condE{s}^\pi[\sum_{t=0}^\infty\gamma^tr(s_t,\pi(s_t))]$, where $\condE{s}^\pi$ denotes expectation w.r.t. the distribution induced by $\pi$ and conditioned on the event $\{s_0=s\}.$  For brevity, we respectively denote the reward and value at time $t$ by $r_t\equiv r(s_t,\pi_t(s_t))$ and $v_t\equiv v(s_t).$  It is known that ${v^\pi=\sum_{t=0}^\infty \gamma^t (P^\pi)^t r^\pi=(I-\gamma P^\pi)^{-1}r^\pi}$, with the component-wise values $[P^\pi]_{s,s'}  \triangleq P(s'\mid s, \pi(s))$ and $[r^\pi]_s \triangleq  r(s,\pi(s))$. Our goal is to find a policy $\pi^*$ yielding the optimal value $v^*$ such that $v^* = \max_\pi (I-\gamma P^\pi)^{-1} r^\pi$.
This goal can be achieved using the three classical operators (with equalities holding component-wise):  
\begin{align}
\forall v,\pi,~  T^\pi v & =  r^\pi +\gamma P^\pi v, \label{def: Tpi} \\
\forall v,~  T v & =  \max_\pi T^\pi v, \\
\forall v,~\G(v)&= \{\pi : T^\pi v = T v\}, \label{def: greedy policy}
\end{align}
where $T^\pi$ is a linear operator, $T$ is the optimal Bellman operator and both $T^\pi$ and $T$ are $\gamma$-contraction mappings w.r.t. the max norm. It is known that the unique fixed points of $T^\pi$ and $T$ are $v^\pi$ and $v^*$, respectively. The set $\G(v)$ is the standard set of 1-step greedy policies w.r.t. $v$. 
Furthermore, given $v^*$, the set $\G(v^*)$ coincides with that of stationary optimal policies. In other words, every policy that is 1-step greedy w.r.t. $v^*$ is optimal and vice versa.


The most known variants of PI are Modified-PI \cite{puterman1978modified} and $\lambda$-PI \cite{lpi}. In both, the evaluation stage of PI is relaxed by performing partial-evaluation, instead of the full policy evaluation. In this work, we will generalize algorithms using both of these approaches.  
Modified PI performs partial evaluation using the $m$-return, $(T^{\pi})^m v$, where $\lambda$-PI uses the $\lambda$-return, $T_\lambda^\pi v$, with $\lambda\in[0,1]$. This operator has the following equivalent forms (see e.g.~\cite{alpi}, p.1182),
\begin{align}
T_\lambda^\pi v &\eqdef (1-\lambda) \sum_{j=0}^\infty \lambda^j (T^\pi)^{j+1} v \label{def: T lambda pi}\\
&= v+ (I-\gamma\lambda P^\pi)^{-1}(T^\pi v - v). \nonumber
\end{align}
These operators correspond to the ones used in the famous TD($n$) and TD($\lambda$) \cite{sutton1998reinforcement},
\begin{align*}
&(T^{\pi})^m v =  \condE{\bigcdot}^{\pi}\left[\sum_{t=0}^{m-1}\gamma^t r(s_t,\pi_t(s_t))+\gamma^m v(s_h)\right],\\
& T_\lambda^\pi v = v + \condE{\bigcdot}^{\pi}\left[\sum_{t=0}^{\infty}(\gamma\lambda)^t (r_t+\gamma v_{t+1}-v_t)\right].
\end{align*}

\section{The $h$-Greedy Policy and $h$-PI}
Let $h\in \mathbb{N} \backslash \{0\}$. An $h$-greedy policy \cite{bertsekas1995neuro,beyond2018efroni} $\pi_h$ outputs the first optimal action out of the sequence of actions solving a non-stationary, $h$-horizon control problem as follows: 
\begin{align}
  &\arg\max_{\pi_0} \max_{\pi_1,..,\pi_{h-1}}  \condE{\bigcdot}^{\pi_0\dots\pi_{h-1}}\left[\sum_{t=0}^{h-1}\gamma^t r(s_t,\pi_t(s_t))+\gamma^h v(s_h)\right] \nonumber\\
  & = \arg\max_{\pi_0} \condE{\bigcdot}^{\pi_0}\left[r(s_0,\pi_0(s_0))+\gamma \left(T^{h-1} v\right)(s_1) \right],  \label{eq_h_greedy_def}
\end{align}
where the notation $\condE{\bigcdot}^{\pi_0\dots\pi_{h-1}}$ corresponds to conditioning on the trajectory induced by the choice of actions $(\pi_0(s_0),\pi_1(s_1),\dots, \pi_{h-1}(s_{h-1}))$ and a starting state $s_0=\bigcdot$.

 As the equality in \eqref{eq_h_greedy_def} suggests that $\pi_h$ can be interpreted as a 1-step greedy policy w.r.t. $T^{h-1}v$. We denote the set of $h$-greedy polices w.r.t $v$ as $\G_h(v)$ and is defined by
 \begin{align*}
\forall v,~\G_h(v)&= \{\pi : T^\pi T^{h-1} v = T^h v\}.
 \end{align*}
 This generalizes the definition of the 1-step greedy set of policies, generalizing, \eqref{def: greedy policy}, and coincides with it for $h=1$.
\begin{figure}
\centering
\def\svgwidth{5.3cm}
\begingroup%
  \makeatletter%
  \providecommand\color[2][]{%
    \errmessage{(Inkscape) Color is used for the text in Inkscape, but the package 'color.sty' is not loaded}%
    \renewcommand\color[2][]{}%
  }%
  \providecommand\transparent[1]{%
    \errmessage{(Inkscape) Transparency is used (non-zero) for the text in Inkscape, but the package 'transparent.sty' is not loaded}%
    \renewcommand\transparent[1]{}%
  }%
  \providecommand\rotatebox[2]{#2}%
  \newcommand*\fsize{\dimexpr\f@size pt\relax}%
  \newcommand*\lineheight[1]{\fontsize{\fsize}{#1\fsize}\selectfont}%
  \ifx\svgwidth\undefined%
    \setlength{\unitlength}{588.8452581bp}%
    \ifx\svgscale\undefined%
      \relax%
    \else%
      \setlength{\unitlength}{\unitlength * \real{\svgscale}}%
    \fi%
  \else%
    \setlength{\unitlength}{\svgwidth}%
  \fi%
  \global\let\svgwidth\undefined%
  \global\let\svgscale\undefined%
  \makeatother%
  \begin{picture}(1,0.68399994)%
    \lineheight{1}%
    \setlength\tabcolsep{0pt}%
    \put(0.1744594,0.5613301){\color[rgb]{0,0,0}\makebox(0,0)[lt]{\lineheight{1.25}\smash{\begin{tabular}[t]{l}$r_{t=0}$\end{tabular}}}}%
    \put(0.044544,0.35421169){\color[rgb]{0,0,0}\makebox(0,0)[lt]{\lineheight{1.25}\smash{\begin{tabular}[t]{l}$\gamma r_{t=1}$\end{tabular}}}}%
    \put(-0.00385586,0.13569517){\color[rgb]{0,0,0}\makebox(0,0)[lt]{\lineheight{1.25}\smash{\begin{tabular}[t]{l}$\gamma^2 v_{t=2}$\end{tabular}}}}%
    \put(0,0){\includegraphics[width=\unitlength,page=1]{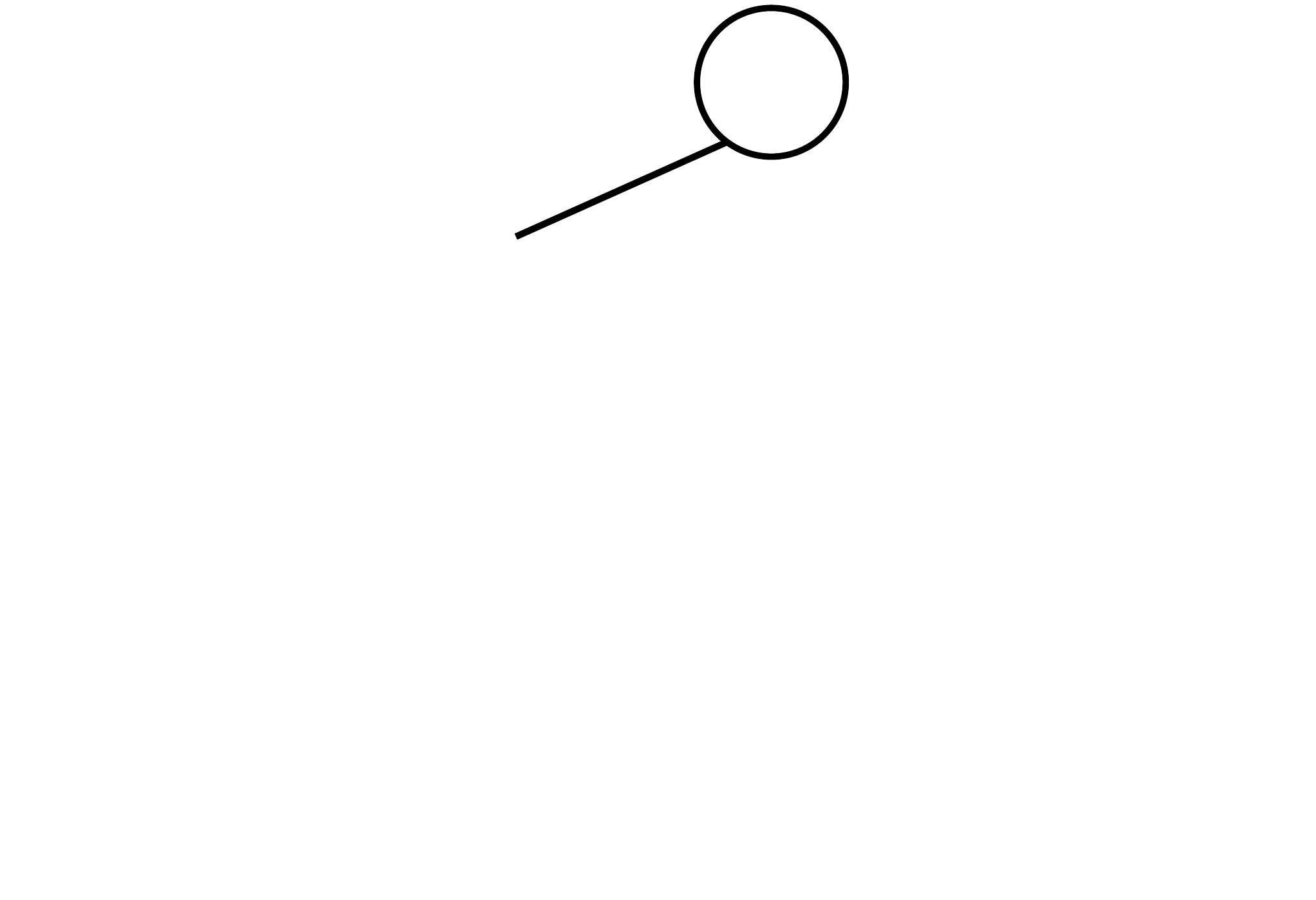}}%
    \put(0.57293138,0.60758265){\color[rgb]{0,0,0}\makebox(0,0)[lt]{\lineheight{1.25}\smash{\begin{tabular}[t]{l}$s$\end{tabular}}}}%
    \put(0,0){\includegraphics[width=\unitlength,page=2]{hTree3.pdf}}%
    \put(0.79132045,0.44476297){\color[rgb]{0,0,0}\makebox(0,0)[lt]{\lineheight{1.25}\smash{\begin{tabular}[t]{l}$s_r$\end{tabular}}}}%
    \put(0.33663013,0.44341203){\color[rgb]{0,0,0}\makebox(0,0)[lt]{\lineheight{1.25}\smash{\begin{tabular}[t]{l}$s_l$\end{tabular}}}}%
  \end{picture}%
\endgroup%

\caption{Obtaining the $h$-greedy policy with a tree-search also outputs $T^{\pi_h}T^{h-1}v$ and $T^{h-1}v$. In this example, the red arrow depicts the $h$-greedy policy. The value at the root's child node $s_l$ is $T^{h-1}v(s_l),$ which corresponds to the optimal blue trajectory starting at $s_l$. The same holds for $s_r$. }
\label{fig: tree search}
\end{figure} 
 
\begin{remark}\label{remark: h optimal value}
The $h$-greedy policy can be obtained by solving the above formulation with DP in linear time (in $h$). Other than returning the policy, the last and one-before-last iterations also return $T^{\pi_h}T^{h-1}v$ and $T^{h-1}v,$ respectively. Another, conceptually similar option would be using Model Predictive Control to solve the planning problem and again retrieve the above values of interest \cite{negenborn2005learning,tamar2017learning}. Given a `nice' mathematical structure, this can be done efficiently. When the model is unknown, finding $\pi_h$ together with  $T^{\pi_h}T^{h-1}v$ and $T^{h-1}v$ is possible with model-free approaches such as Q-learning \cite{jin2018q}. Alternatively, $\pi_h(s)$ can be retrieved using a tree-search of depth $h$, starting at root $s$ (see Figure \ref{fig: tree search}). The search again returns $T^{\pi_h}T^{h-1}v$  and $T^{h-1}v$ ``for free'' as the values at the root and its descendant nodes.  While the tree-search complexity in general is exponential in $h$, sampling can be used. Examples for such sampling-based tree-search methods are MCTS \cite{browne2012survey} and Optimistic Tree Exploration \cite{munosbook14}.
\end{remark} 

\begin{algorithm}
	\caption{$h$-PI}
	\label{alg:hPI}
	\begin{algorithmic}
		\STATE {\bf Initialize:} $h \in \mathbb{N} \setminus \{0\},~v_0=v^{\pi_0} \in \mathbb{R}^{|\S|}$
		\WHILE{$v_k$ changes}
		\STATE $\pi_{k} \gets \pi \in \G_h(v)$
		\STATE $v_{k+1} \gets v^{\pi_k}$
		\STATE $~~ k ~~ ~\gets k+1$
		\ENDWHILE
		\STATE {\bf Return $\pi,v$}
	\end{algorithmic}
\end{algorithm}

As was discussed in \cite{bertsekas1995neuro,beyond2018efroni}, one can use the $h$-greedy policy to derive a policy-iteration procedure called $h$-PI (see Algorithm \ref{alg:hPI}). In it, the 1-step greedy policy from PI is replaced with the $h$-greedy policy. This algorithm iteratively calculates an $h$-step greedy policy with respect to $v$, and then performs a complete evaluation of this policy. Convergence is guaranteed after $O(h^{-1})$ iterations \cite{beyond2018efroni}.

\section{$h$-Greedy Consistency}
The $h$-greedy policy w.r.t $v^\pi$ is strictly better than $\pi$, i.e., $v^{\pi_h}\geq v^\pi$ \cite{bertsekas1995neuro,beyond2018efroni}. Using this property for proving convergence of an algorithm requires the algorithm to perform exact value estimation, which can be a hard task. Instead, in this work, we replace the less practical exact evaluation with partial evaluation; this comes with the price of more challenging analysis. Tackling this more intricate setup, we identify a key property required for the analysis to hold. We refer to it as \emph{$h$-greedy consistency}. It will be central to all proofs in this work.
\begin{defn}\label{defn: h consistent}
A pair of value function and policy $(v,\pi)$ is $h$-greedy consistent if $T^{\pi}T^{h-1} v \geq T^{h-1}v$.
\end{defn}

In words, $(v,\pi)$ is $h$-greedy consistent if $\pi$ `improves', component-wise, the value $T^{h-1}v$. 
Since relaxing the evaluation stage comes with the $h$-greedy consistency requirement, the following question arises: while dispatching an algorithm, what is the price ensuring $h$-greedy consistency per each iteration? As we will see in the coming sections, it is enough to ensure $h$-greedy consistency only for the first iteration of our algorithms. For the rest of the iterations it holds by construction and is shown to be guaranteed in our proofs. Thus, by only initializing to an $h$-greedy consistent $(v_0,\pi_0)$, we enable guaranteeing the convergence of an algorithm that performs partial evaluation instead of exact in each its iterations. Ensuring consistency for the first iteration is straightforward, as is explained in the following remark.


\begin{remark}\label{remark: h greedy consistancy is easier}
Choosing $(v,\pi)$ which is $h$-greedy consistent can be done, e.g., by choosing $v = \frac{R_{\min}}{1-\gamma}$ (i.e., set every entrance of $v$ to the minimal possible accumulated reward) and ${\pi=\pi_h\in \G_h(v).}$ Furthermore, for any value-policy, $(\bar{v},\pi)$, that is not $h$-greedy consistent, let 
$$\Delta= \frac{ \max_s \left(T^{h-1}\bar{v} -T^{\pi}T^{h-1}\bar{v}\right)(s)}{\gamma^{h-1}(1-\gamma)}>0,$$
and set ${v = \bar{v}- \Delta}$. Then, $(v,\pi)$ is $h$-greedy consistent. This is a generalization to the construction given for $h=1$ (see \cite{bertsekas1995neuro}, p. 46).
\end{remark}

$h$-greedy consistency is an $h$-step generalization of a notion already introduced in previous works on 1-step-based PI schemes with partial evaluation. The latter are known as `optimistic' PI schemes and include Modified PI and $\lambda$-PI \cite{bertsekas1995neuro}. There, the initial value-policy pair is assumed to be 1-greedy consistent, i.e. $T^{\pi_1}v_0 \geq v_0$, e.g., \cite{bertsekas1995neuro}, p. 32 and 45, \cite{bertsekas2011approximate}, p. 3, \cite{puterman1978modified}[Theorem 2]. This property served as an assumption on the pair $(v_0,\pi_1)$. 



To further motivate our interest in Definition~\ref{defn: h consistent}, in the rest of the section we give two results that would be used in proofs later but are also insightful on their own. The following lemma gives that $h$-greedy consistency implies a sequence of value-function partial evaluation relations (see proof in Appendix~\ref{supp: help value improvement}).
\begin{lemma}\label{lemma: help value improvement}
Let $(v,\pi)$ be $h$-greedy consistent. Then,
\begin{align*}
T^\pi T^{h-1}v \leq \dotsm \leq (T^{\pi})^l T^{h-1}v \leq  \dotsm \leq v^\pi.
\end{align*}
\end{lemma}
The result shows that $v^\pi$ is strictly bigger than $T^{h-1}v$. This property holds when $v=v^{\pi'}$, i.e., when $v$ is an exact value of some policy and was central in the analysis of $h$-PI \cite{beyond2018efroni}. However, as Lemma \ref{lemma: help value improvement} suggests, we only need $h$-greedy consistency, which is easier to have than estimating the exact value of a policy (see Remark \ref{remark: h greedy consistancy is easier}).



The next result shows that if $\pi$ is taken to be the $h$-greedy policy, using partial evaluation results in a $\gamma^h$ contraction toward the optimal value $v^*$ (see proof in Appendix~\ref{supp: help value improvement2}).

\begin{proposition}\label{proposition: help value improvement2}
Let $v$ and $\pi_h\in \G_h(v)$ be s.t. $(v,\pi_h)$ is $h$-greedy consistent. Then, for any $m\geq 1 \mbox{ and } \lambda\in[0,1],$
\begin{align*}
&||v^*-(T^{\pi_h})^m T^{h-1} v ||_\infty \leq \gamma^h ||v^*-v ||_\infty \mbox{ and }\\
&||v^*-T^{\pi_h}_\lambda T^{h-1} v ||_\infty \leq \gamma^h ||v^*-v ||_\infty.
\end{align*}
\end{proposition}

In \cite{beyond2018efroni}[Lemma 2], a similar contraction property was proved and played a central role in the analysis of the corresponding $h$-PI algorithm. Again, there, the requirement was $v=v^{\pi'}.$ Instead, the above result requires a weaker condition: $h$-greedy consistency of $(v,\pi_h)$. 

%

\section{The $h$-Greedy Policy Alone is Not Sufficient For Partial Evaluation}\label{sec: wrong hmPI hlambdaPI}
A more practical version of $h$-PI (Algorithm~\ref{alg:hPI}) would involve the $m$- or $\lambda$-return w.r.t.  $v_{k}$ instead of the exact value. This would correspond to the update rules: 
\begin{align}
	& \pi_{k} \gets  \arg\max_{\pi} T^{\pi}T^{h-1}v_k, \label{eq: bad alg policy upadte}\\
	& v_{k+1} \gets  (T^{\pi_k})^m v_k \mbox{ or }  v_{k+1} \gets  T^{\pi_k}_\lambda v_k \label{eq: bad alg value upadte}.
\end{align}
Indeed, this would relax the evaluation task to an easier task than full policy evaluation.

The next theorem suggests that for $\pi_h \in \G_h(v),$ even if $(v,\pi_h)$ is $h$-greedy consistent, the procedure \eqref{eq: bad alg policy upadte}-\eqref{eq: bad alg value upadte}  does not necessarily contract toward the optimal policy, unlike the form of update in Proposition \ref{proposition: help value improvement2}. To see that, note that both $\gamma^m + \gamma^h$ and $\frac{\gamma(1-\lambda)}{1-\lambda\gamma} + \gamma^h$ can be larger than 1. 
\begin{theorem}\label{thm: contraction coefficient}
Let $h>1,m\geq 1,\mbox{ and }\lambda\in [0,1]$. Let $v$ be a value function and $\pi_h\in \G_h(v)$ s.t.  $(v,\pi_h)$ is $h$-greedy consistent (see Definition~\ref{defn: h consistent}). Then,
\begin{align}
&||v^* - (T^{\pi_h})^m v||_\infty  \le (\gamma^m+\gamma^h)|| v^* - v ||_\infty, \label{eq: non-contraction m}\\
&||v^* - T^{\pi_h}_\lambda v||_\infty \le \left(\frac{\gamma(1-\lambda)}{1-\lambda\gamma} + \gamma^h \right) || v^* - v ||_\infty. \label{eq: non-contraction lambda}
\end{align}
Additionally, there exist a $\gamma$-discounted MDP, value function $v,$ and policy $\pi_h\in \G_h(v)$ s.t. $(v,\pi_h)$ is $h$-greedy consistent, for which \eqref{eq: non-contraction m} and \eqref{eq: non-contraction lambda} hold with equality.
\end{theorem}
The proof of the first statement is given in Appendix~\ref{supp: contraction coefficicent}, and the proof of the second statement is as follows.
\begin{proof}[Proof of second statement in Theorem~\ref{thm: contraction coefficient}]
	We prove this by constructing an example. Fix $h>1$ and consider the corresponding 4-state MDP in Figure \ref{fig: hm doesnt contract}. Let $v$ be
	$
	v(s_0)=v(s_2)=v(s_3)=0,\ v(s_1)=-\frac{1}{1-\gamma}.
	$
	Also, let $\pi_h\in \G_h(v)$. For this choice, observe that ${T^{h-1}v \leq T^{\pi_h}T^{h-1} v}$, i.e., $(v,\pi_h)$ is $h$-greedy consistent.
	The optimal policy from state $s_{0}$ is to choose the action `up'. Thus, it is easy to see that, ${v^*(s_{0})=v^*(s_{3})=\frac{1}{1-\gamma}}$, and in the remaining of states it is easy to observe that ${v^*(s_1)=v^*(s_2)=0}$.
	
	Now, see that for any $h>1$
	$
	\left(T^{h-1}v\right)(s_1)=\left(T^{h-1}v\right)(s_2)=0,
	\left(T^{h-1}v\right)(s_3)=\frac{1-\gamma^{h-1}}{1-\gamma}.
	$
	Thus, the $h$-greedy policy (by using \eqref{eq_h_greedy_def}) is contained in the following set of actions
	$
	\pi_{h}(s_{0})\in \{ \mathrm{right,up}\},\pi_{h}(s_{1})\in \{ \mathrm{right,stay}\},
	\pi_{h}(s_{2}), \pi_{h}(s_{3}) \in \{ \mathrm{stay}\}.
	$
	For example, we see that taking the action `stay' or `right' from state $s_{1}$ and then obtain $T^{h-1}v$ have equal value:
	\begin{align*}
	&r(s_1,\mathrm{`stay'})+\gamma (T^{h-1}v)(s_1) \\
	= &r(s_1,\mathrm{`right'})+\gamma (T^{h-1}v)(s_2)=0.
	\end{align*}
	Let us choose an $h$-greedy policy, $\pi_h$, of the form:
	$
	\pi_{h}(s_{0})=  \mathrm{right},\pi_{h}(s_{1})=  \mathrm{stay},\pi_{h}(s_{2})=  \mathrm{stay}.
	$
	Thus, from state $s_0$, the $m$-return has the value
	\begin{align*}
	&\left( (T^{\pi_{h}})^m v \right) (s_{0}) = \sum_{i=0}^{m-1} \gamma^t r(s_i,\pi_h(s_i)) + \gamma^m v(s_{i=m}) \\
	&=\frac{1-\gamma^m-\gamma^h}{1-\gamma}+\sum_{i=1}^{m-1} \gamma^i\cdot 0 + \gamma^m \left(-\frac{1}{1-\gamma} \right)\\
	& = \frac{1-\gamma^m-\gamma^h}{1-\gamma}
	\end{align*}
	We thus have that
	\begin{align}
	&|| v^*-(T^{\pi_h})^m v ||_\infty = \left | v^*(s_{1,0})-(T^{\pi_{h}})^m v(s_{1,0})  \right| \nonumber\\
	&=\frac{1}{1-\gamma} + \frac{\gamma^m+\gamma^h-1}{1-\gamma}=(\gamma^m+\gamma^h)\frac{1}{1-\gamma}\label{eq: v_star and m return}
	\end{align}
	
	It is also easy to see that $|| v^*-v ||_\infty =\frac{1}{1-\gamma}$. By using~\eqref{eq: v_star and m return},
	\begin{align*}
	|| v^*-(T^{\pi_h})^m v ||_\infty = (\gamma^m+\gamma^h)|| v^*-v ||_\infty,
	\end{align*}
	which concludes the tightness result on the first result in Theorem \ref{thm: contraction coefficient}. The tightness proof of \eqref{eq: non-contraction lambda} easily follows using the same construction as above; for details see Appendix~\ref{supp: contraction coefficicent}.

\end{proof}
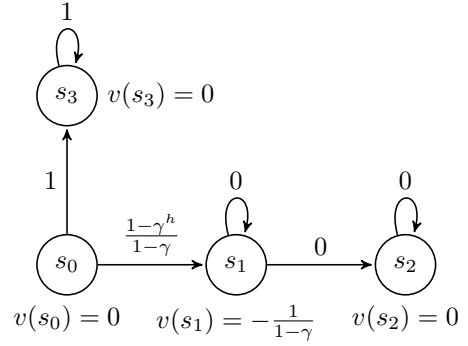
\begin{figure}
	\begin{center}
		\begin{tikzpicture}[->,>=stealth',shorten >=1pt,auto,node distance=2.25cm,
		semithick, state/.style={circle, draw, minimum size=0.8cm}]
		\tikzstyle{every state}=[thick]
		]
		
		\node[state] (S0) [label=below:{$v(s_{0})=0$}]{$s_0$};
		\node[state] (S1) [right of=S0,label=below:{$v(s_{1})=-\frac{1}{1-\gamma}$}] {$s_1$};
		\node[state] (S2) [right of=S1,label=below:{$v(s_{2})=0$}] {$s_2$};
		\node[state] (S3) [above of=S0,label=right:{$v(s_{3})=0$}] {$s_3$};

		\path (S0) edge              node {$\frac{1-\gamma^h}{1-\gamma}$} (S1)
		(S1) edge              node {$0$} (S2)
		(S1) edge   [loop above] node[pos=0.05,left]{} node {$0$}  (S1)
		(S3) edge   [loop above] node[pos=0.05,left]{} node {$1$}  (S3)
		(S0) edge   node {$1$}  (S3)
		(S2) edge   [loop above] node[pos=0.05,left]{} node {$0$}  (S2);
		\end{tikzpicture}
	\end{center}
	\caption{The MDP used in the proof of Theorem~\ref{thm: contraction coefficient}. NC-$hm$-PI and NC-$h\lambda$-PI may result in a new value that does not contract toward $v^*$.}
	\label{fig: hm doesnt contract}
\end{figure}

As discussed above, Theorem~\ref{thm: contraction coefficient} suggests that the `naive' partial-evaluation scheme would not necessarily lead to contraction toward the optimal value, especially for small values of $h,m,\lambda$ and large $\gamma$; these are often values of interest. Moreover, the second statement in the theorem contrasts with the known result for $h=1$, i.e., Modified PI and $\lambda$-PI. There, a $\gamma$-contraction was shown to exist \cite{alpi}[Proposition 8] and \cite{puterman1978modified}[Theorem 2]. 

From this point onwards, we shall refer to the algorithms given in \eqref{eq: bad alg policy upadte}-\eqref{eq: bad alg value upadte} and discussed in this section as Non-Contracting (NC)-$hm$-PI and NC-$h\lambda$-PI.



\section{Backup the Tree-Search Byproducts}\label{sec: hm-PI and hlambda-PI}

\begin{figure*}
\noindent\makebox[\textwidth][c]{
\begin{minipage}[t]{0.45\textwidth}
\centering
\begin{algorithm}[H]
	\caption{$hm$-PI}
	\label{alg:hmPI}
	\begin{algorithmic}
		\STATE {\bf Initialize:} $h,m \in \mathbb{N} \setminus \{0\},~v \in \mathbb{R}^{|\S|}$
		\WHILE{stopping criterion is false}
		\STATE $\pi_{k+1} \gets \pi\in \approxhgreedy{\delta_{k+1}}{v_k}$
		\STATE $v_{k+1} \gets (T^{\pi_{k+1}})^m T^{h-1}v_k + \epsilon_k$
		\STATE $~~ k ~~ ~\gets k+1$
		\ENDWHILE
		\STATE {\bf Return $\pi,v$}
	\end{algorithmic}
\end{algorithm}
\end{minipage}
\hspace{0.4cm}
\begin{minipage}[t]{0.45\textwidth}
\centering
 \begin{algorithm}[H]
	\caption{$h\lambda$-PI}
	\label{alg:hlambdaPI}
	\begin{algorithmic}
		\STATE {\bf Initialize:} $h \in \mathbb{N} \setminus \{0\},~\lambda\in[0,1],~v \in \mathbb{R}^{|\S|}$
		\WHILE{stopping criterion is false}
		\STATE $\pi_{k+1} \gets \pi\in \approxhgreedy{\delta_{k+1}}{v_k}$
		\STATE $v_{k+1} \gets T^{\pi_{k+1}}_\lambda T^{h-1}v_k + \epsilon_k$
		\STATE $~~ k ~~ ~\gets k+1$
		\ENDWHILE
		\STATE {\bf Return $\pi,v$}
	\end{algorithmic}
\end{algorithm}
\end{minipage}}
\end{figure*}

In the previous section, we proved that partial evaluation using the backed-up value function $v$, as given in \eqref{eq: bad alg policy upadte}-\eqref{eq: bad alg value upadte}, is not necessarily a process converging toward the optimal value. In this section, we propose a natural respective fix: back up the value $T^{h-1}v$ and perform the partial evaluation w.r.t. it. In the noise-free case this is motivated by Proposition~\ref{proposition: help value improvement2}, which reveals a $\gamma^h$-contraction per each PI iteration.


We now introduce two new algorithms that relax $h$-PI's (from Algorithm~\ref{alg:hPI}) exact policy evaluation stage to the more practical $m$- and $\lambda$-return partial evaluation. Notice that $hm$-PI can be interpreted as iteratively performing $h-1$ steps of Value Iteration and one step of Modified PI \cite{puterman1978modified}, whereas instead of the latter, $h\lambda$-PI performs one step of $\lambda$-PI \cite{lpi}.

Our algorithms also account for noisy updates in both the improvement and evaluation stages. For that purpose, we first define the following approximate improvement operator.
\begin{defn}
	For $\hat{\delta} \in \mathbb{R}^{|\mathcal{S}|}_+,$ let $\approxhgreedy{\hat{\delta}}{v}$ be the approximate {$h$-greedy} set of policies w.r.t. $v$ with error $\hat{\delta},$ s.t. for $\pi \in \approxhgreedy{\hat{\delta}}{v},$ 
	$T^\pi T^{h-1} v \geq T^hv - \hat{\delta}.$
\end{defn}
Additionally, the algorithms assume additive $\hat{\epsilon} \in \mathbb{R}^{|\mathcal{S}|}$ error in the evaluation stage. We call them $hm$-PI and $h\lambda$-PI and present them in Algorithms~\ref{alg:hmPI} and \ref{alg:hlambdaPI}. As opposed to the non-contracting update discussed in Section~\ref{sec: wrong hmPI hlambdaPI}, the evaluation stage in these algorithms uses $T^{h-1}v$. 

We now provide our main result, demonstrating a $\gamma^h$-contraction coefficient for both $hm$-PI and $h\lambda$-PI.

The proof technique builds upon the previously introduced \emph{invariance argument} (see \cite{beyond2018efroni}, proof of Theorem~9). This enables working with a more convenient, shifted noise sequence. Thereby, we construct a shifted noise sequence s.t. the value-policy pair $(v_k,\pi_{k+1})$ in each iteration is $h$-greedy consistent  (see Definition \ref{defn: h consistent}). We thus also eliminate the $h$-greedy consistency assumption on the initial $(v_0,\pi_1)$ pair, which appears in previous works (see Remark \ref{remark: h greedy consistancy is easier}). Specifically, we shift $v_0$ by $\Delta_0$; the latter quantifies how `far' $(v_0,\pi_1)$ is from being $h$-greedy consistent. Notice our bound explicitly depends on  $\Delta_0$.  The provided proof is simpler and shorter than in previous works (e.g. \cite{alpi}). We believe that the proof technique presented here can be used as a general `recipe' for proving newly-devised PI procedures that use partial evaluation with more ease.

%
\begin{theorem}\label{adp}
  Let $h,m \in \mathbb{N} \setminus \{0\}, \lambda\in[0,1]$. 
%
  For noise sequences $\{\epsilon_k\}$ and $\{\delta_k\}$, $\|\epsilon_k\|_\infty \le \epsilon$ and $\|\delta_k\|_\infty \le \delta.$ Let
  $$\Delta_0= \max \{0,\frac{ \max_s \left(T^{h-1}v_0 -T^{\pi_1}T^{h-1}v_0\right)(s)}{\gamma^{h-1}(1-\gamma)}\}.$$

  Then,
  \begin{align*}
  \|v^* - v^{\pi_{k+1}} \|_\infty  \le ~ & \gamma^{kh}|| v^*-(v_0-\Delta_0) ||_\infty  \\ &+ \frac{(2 \gamma^h \epsilon + \delta)(1-\gamma^{kh})}{(1-\gamma)(1-\gamma^h)}
  \end{align*}
  and hence $ \lim\sup_{k \to \infty} \|v^* - v^{\pi_k} \|_\infty  \le\! \frac{2 \gamma^h \epsilon + \delta}{(1-\gamma)(1-\gamma^h)}.$

\end{theorem}

\begin{proof}

We start with the invariance argument. Consider the process with the alternative error in the evaluation stage, $\epsilon_k'=\epsilon_k-C_k e$, where  $C_k = \frac{\max \delta_{k+1}+ \gamma^{h-1}\max \epsilon_k -\gamma^h \min \epsilon_k}{\gamma^{h-1}(1-\gamma)}$, and $e$ a vector of `ones' of dimension $|\mathcal{S}|$. Next, given initial value $v_0$, let $v'_0=v_0-\Delta_0$. As described in Remark \ref{remark: h greedy consistancy is easier}, this transformation makes $(v_0',\pi_1)$ $h$-greedy consistent. Since the greedy policy is invariant for an addition of a constant, i.e., for $\alpha \in \mathbb{R},~\G_h(v+\alpha e)=\G_h(v)$, and since $T^h(v+\alpha e)=T^h v+\gamma^h\alpha$, we have that the \emph{sequence of policies generated is invariant} for the offered transformation.

Next, we use Lemma~\ref{lemma: invariant property noisy}, which gives that the choice of $C_k$ leads to a sequence of pairs of $h$-greedy consistent policies and values in every iteration.
Thus, we can now continue with simpler analysis than in \cite{beyond2018efroni}. 

At this stage of the proof we focus on $hm$-PI. Define $d_k \eqdef v^* -(v'_k-\epsilon_k')$ for $k\geq 1$, and $d_0 \eqdef v^* -v'_0$. We get
\begin{align}
d_{k+1} &=  v^* -(T^{\pi_{k+1}})^m T^{h-1} v'_{k} \label{eq: hm proof 1}\\
&\leq  v^* -T^{\pi_{k+1}} T^{h-1} v'_{k} \label{eq: hm proof 2}\\
&\leq  v^* -T^{h} v'_{k} +\max \delta_{k+1}\nonumber\\
&\leq  (T^{\pi_*})^{h} v^* -T^{h} ( v'_{k}-\epsilon_{k}') -\gamma^h \min \epsilon_k' +\max \delta_{k+1}\nonumber\\
&\leq  (T^{\pi_*})^{h}v^* -(T^{\pi_*})^{h} ( v'_{k}-\epsilon_{k}') -\gamma^h \min \epsilon_k' +\max \delta_{k+1}\nonumber\\
&=  \gamma^h(P^{\pi_*})^{h}(v^* -( v'_{k}-\epsilon_{k}')) -\gamma^h \min \epsilon_k' +\max \delta_{k+1}\nonumber\\
&=  \gamma^h(P^{\pi_*})^{h}d_{k} -\gamma^h \min \epsilon_k' +\max \delta_{k+1}. \nonumber
\end{align}
The second relation holds by applying Lemma \ref{lemma: help value improvement} on $\pi_{k+1}$ and $v_k$ which are $h$-consistent. Furthermore, by using the form of $C_k$ and simple algebraic manipulations it can be shown that ${-\gamma^h \min \epsilon_k' +\max \delta_{k+1}\leq \frac{2\gamma^h\epsilon+\delta}{1-\gamma}}$. Thus,
\begin{align}
d_{k+1} \leq \gamma^h(P^{\pi_*})^{h}d_{k} + \frac{2\gamma^h\epsilon+\delta}{1-\gamma}.
\end{align}
Iteratively applying the above relation on $k$, we get that
\begin{align}
d_k &\leq \gamma^{kh}(P^{\pi_*})^{kh}d_{0} + \frac{(2\gamma^h\epsilon+\delta)(1-\gamma^{kh})}{(1-\gamma)(1-\gamma^h)} \nonumber\\ 
&\leq \gamma^{kh}||d_{0}||_\infty + \frac{(2\gamma^h\epsilon+\delta)(1-\gamma^{kh})}{(1-\gamma)(1-\gamma^h)}.\label{eq: recursive equation}
\end{align}
To conclude the proof for $hm$-PI notice that 
$
v^*-v^{\pi_{k+1}} \leq v^* - (v'_k-\epsilon_k') = d_k,
$
which holds due to the second claim in Lemma~\ref{lemma: invariant property noisy} combined with Lemma~\ref{lemma: help value improvement}. Since the LHS is positive, we can apply the max norm on the inequality and use \eqref{eq: recursive equation}:
\begin{align*}
 ||v^*-v^{\pi_{k+1}}||_\infty \leq \gamma^{kh}||d_{0}||_\infty + \frac{(2\gamma^h\epsilon+\delta)(1-\gamma^{kh})}{(1-\gamma)(1-\gamma^h)}.
\end{align*} 
Since $d_{0}=v^*-v_0'= v^*-(v_0-\Delta_0)$, we obtain the first claim for $hm$-PI. Taking the limit easily gives the second claim, again for $hm$-PI:
\begin{align*}
\lim_{k\rightarrow\infty} ||v^*-v^{\pi_{k+1}}||_\infty \leq \frac{2\gamma^h\epsilon+\delta}{(1-\gamma)(1-\gamma^h)}.
\end{align*} 

The convergence proof for $h\lambda$-PI is identical to that of $hm$-PI, except for a minor change: the transition from \eqref{eq: hm proof 1} to \eqref{eq: hm proof 2} holds due to the following argument: 
\begin{align*}
d_{k+1} &\leq  v^* -(1-\lambda)\sum_i \lambda^i(T^{\pi_{k+1}})^{i+1} T^{h-1} v'_{k}\\
 &\leq  v^* -(1-\lambda)\sum_i \lambda^i T^{\pi_{k+1}} T^{h-1} v'_{k} \\
 &= v^* - T^{\pi_{k+1}} T^{h-1}v'_k, 
\end{align*}
where the second relation holds by applying Lemma \ref{lemma: help value improvement}. This can be used since $\pi_{k+1}$ and $v'_k$ are $h$-greedy consistent according to Lemma \ref{lemma: invariant property noisy}. This exemplifies the advantage of using the notion of $h$-greedy consistency in our proof technique.
\end{proof}
Thanks to using $T^{h-1}v$ in the evaluation stage, Theorem~\ref{adp} guarantees a convergence rate of $\gamma^h$ -- as to be expected when using a $T^h$ greedy operator. Compared to directly using $v$ as is done in Section~\ref{sec: wrong hmPI hlambdaPI}, this is a significant improvement since the latter does not even necessarily contract.

A possibly more `natural' version of our algorithms would back up the value of the root node instead of its descendants. The following remark extends on that.
\begin{remark}\label{remark: root backup}
Consider a variant of $hm$-PI and $h\lambda$-PI, which backs-up $T^{\pi_{k+1}}T^{h-1}v_k$ instead of $T^{h-1}v_k.$ Namely, in this variant, the evaluation stage for $hm$-PI (Algorithm~\ref{alg:hmPI}) is
\begin{align*}
v_{k+1} \gets (T^{\pi_{k+1}})^{m-1} (T^{\pi_{k+1}} T^{h-1}v_k) + \epsilon_k,
\end{align*}
and for $h \lambda$-PI (Algorithm~\ref{alg:hlambdaPI}) it is
\begin{align*}
v_{k+1} \gets \bar{T}^{\pi_{k+1}}_\lambda (T^{\pi_{k+1}} T^{h-1}v_k) + \epsilon_k.
\end{align*} 
The latter is (see Appendix \ref{supp: root backup})
$
{\bar{T}_\lambda^\pi v \eqdef (1-\lambda) \sum_{j=0}^\infty \lambda^j (T^\pi)^{j} v}
= {v+ \lambda(I-\gamma\lambda P^\pi)^{-1}(T^\pi v - v)}
$
 -- a variation of the $\lambda$-return operator from \eqref{def: T lambda pi}, in which $T^{\pi}$ is raised to the power of $j$ and not $j+1$. The performance of these algorithms is equivalent to that of the original $hm$-PI and $h\lambda$-PI, as given in Theorem~\ref{adp}, since
\[
(T^{\pi})^{m-1} T^{\pi} = (T^{\pi})^{m}~\mbox{ and }~\bar{T}^{\pi}_\lambda T^{\pi} = T^\pi_\lambda.
\]
Yet, implementing them is potentially easier in practice, and can be considered more `natural' due to the backup of the root optimal value rather its descendants.
\end{remark}

\section{Relation to Existing Work}
In the context of related theoretical work, we find two results necessitating a discussion. The first is the performance bound of Non-Stationary Approximate Modified PI (NS-AMPI) \cite{lesner2015non}[Theorem 3]. Compared to it, Theorem~\ref{adp} reveals two improvements. First, it gives that $hm$-PI is less sensitive to errors; our bound's numerator  has $\gamma^{h}$ instead of $\gamma$. Second, in each iteration, $hm$-PI requires storing a single policy in lieu of $h$ policies as in NS-AMPI. This makes $hm$-PI significantly more memory efficient.  Nonetheless, there is a caveat in our work compared to \cite{lesner2015non}. In each iteration, we require to approximately solve an $h$-finite-horizon problem, while they require solving approximate $1$-step greedy problem instances. 

The second relevant theoretical result is the performance bound of a recently introduced MCTS-based RL algorithm \cite{pmlr-v80-jiang18a}[Theorem 1]. There, in the noiseless case there is no guarantee for convergence to the optimal policy\footnote{The bound in \cite{pmlr-v80-jiang18a}[Theorem 1] is not necessarily $0$ for $\epsilon=0$ since $B_\gamma,B'_\gamma,\mathbb{D}_0$ and $\mathbb{D}_1$ do not depend on the error and, generally, are not $0$.}. 
Contrarily, in our setup, with
 $\delta=0$ and $\epsilon=0$ both $hm$-PI and $h\lambda$-PI converge to the optimal policy.

Next, we discuss related literature on empirical studies and attempt to explain observations there with the results of this work.
In \cite{baxter1999tdleaf,veness2009bootstrapping,lanctot2014monte} the idea of incorporating the optimal value from the tree-search was experimented with. Most closely related to our synchronous setup is that in \cite{baxter1999tdleaf}. There, motivated by practical reasons, the authors introduced and evaluated both  NC $h\lambda$-PI and $h\lambda$-PI, which they respectively call TD-directed$(\lambda)$ and TDLeaf$(\lambda).$ 
Specifically,  TD-directed$(\lambda)$ and TDLeaf$(\lambda)$ respectively back up $v$ and $T^{\pi_h} T^{h-1}v$. As Remark \ref{remark: h optimal value} suggests, $T^{\pi_h} T^{h-1}v$ can be extracted directly from the tree-search, as is also pointed out in \cite{baxter1999tdleaf}. Interestingly, the authors show that TDLeaf$(\lambda)$  outperforms TD-directed$(\lambda)$. Indeed, Theorem \ref{thm: contraction coefficient} sheds light on this phenomenon.

Lastly, a prominent takeaway message from Theorems~\ref{thm: contraction coefficient} and \ref{adp}  is that AlphaGoZero \cite{silver2017mastering,silver2017mastering2} can be potentially improved. This is because in \cite{silver2017mastering}, the authors do not back up the optimal value calculated from the tree search. As their approach relies on PI (and specifically resembles to $h$-PI), our analysis, which covers noisy partial evaluation, can be beneficial even in the practical setup of AlphaGoZero.

\section{Experiments} \label{sec: experiments}

\begin{figure}
\includegraphics[scale=0.33]{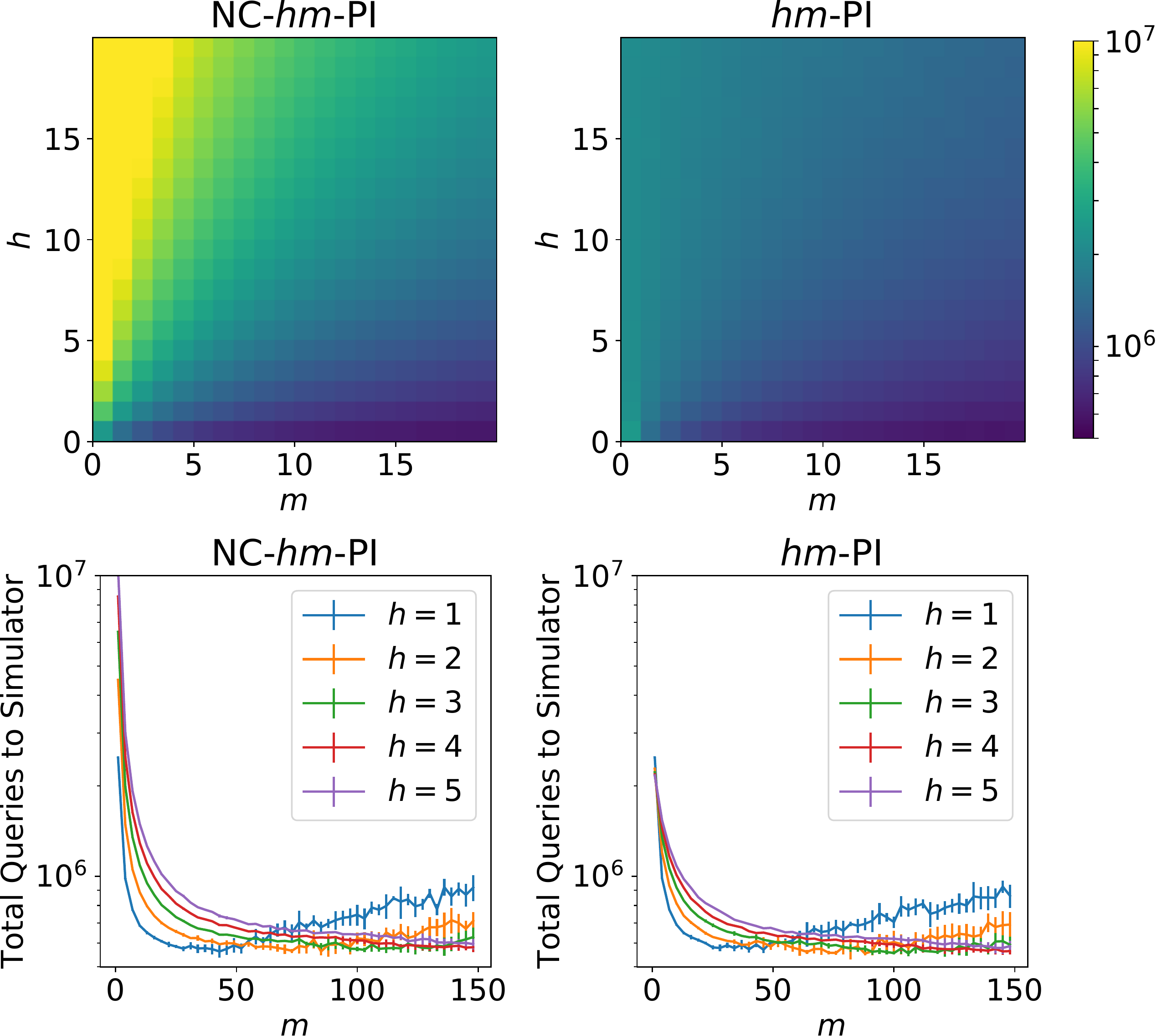}
\caption{(Top) Noiseless NC-$hm$-PI and $hm$-PI convergence time as function of $h$ and $m$. (Bottom) Noiseless NC-$hm$-PI and $hm$-PI convergence time as function of a wide range of $m,$ for several values of $h$. In both figures, the standard error is less than $\% 2$ of the mean.}
\label{fig: noisless res}
\end{figure}

In this section, we empirically study NC-$hm$-PI  (Section~\ref{sec: wrong hmPI hlambdaPI}) and $hm$-PI (Section~\ref{sec: hm-PI and hlambda-PI})  in the exact and approximate cases. Additional results can also be found in Appendix \ref{supp: more results}. Our experiments demonstrate the practicalities of Theorem~\ref{thm: contraction coefficient} and \ref{adp}, even in the simple setup considered here.

We conducted our simulations on a simple $N \times N$ deterministic grid-world problem with $\gamma=0.97$, as was done in \cite{beyond2018efroni}. The action set is \{`up',`down',`right',`left',`stay'\}. In each experiment, a reward $r_g=1$ was placed in a random state while in all other states the reward was drawn uniformly from $[-0.1,0.1]$. In the considered problem there is no terminal state. Also, the entries of the initial value function are drawn from $\mathcal{N}(0,1)$. We ran the algorithms and counted the \emph{total} number of calls to the simulator. Each such ``call'' takes a state-action pair $(s,a)$ as input, and returns the current reward and next (deterministic) state. Thus, it quantifies the total running time of the algorithm, and not the total number of iterations.

We begin with the noiseless case, in which $\epsilon_k$ and $\delta_k$ from Algorithm~\ref{alg:hmPI} are $0$. While varying $h$ and $m$, we counted the total number of queries to the simulator until convergence, which defined as $||v^*-v_k ||_\infty\leq 10^{-7}$.  Figure~\ref{fig: noisless res} exhibits the results. In its top row, the heatmaps give the convergence time for equal ranges of $h$ and $m$. It highlights the suboptimality of NC-$hm$-PI compared to $hm$-PI.
As expected, for $h=1,$ the results coincide for NC-$hm$-PI and $hm$-PI since the two algorithms are then equivalent. For $h>1$, the performance of NC-$hm$-PI significantly deteriorates up to an order of magnitude compared to $hm$-PI.  However, the gap between the two becomes less significant as $m$ increases. This can be explained with Theorem \ref{thm: contraction coefficient}: increasing $m$ in NC-$hm$-PI drastically shrinks $\gamma^m$ in \eqref{eq: non-contraction m} and brings the contraction coefficient closer to $\gamma^h$, which is that of $hm$-PI. In the limit $m\rightarrow \infty$ both algorithms become $h$-PI. 

The bottom row in Figure~\ref{fig: noisless res} depicts the convergence time in 1-d plots for several small values of $h$ and a large range of $m$. It highlights the tradeoff in choosing $m$. As $h$ increases, the optimal choice of $m$ increases as well. 
Further rigorous analysis of this tradeoff in $m$ versus $h$ is an intriguing subject for future work.

Next, we tested the performance of NC-$hm$-PI and $hm$-PI in the presence of evaluation noise.
Specifically, ${\forall  k, s\in \mathcal{S},\ \epsilon_k(s)\sim U(-0.3,0.3)}$ and $ \delta_k(s)=0$. For NC-$hm$-PI, the noise was added according to ${v_{k+1} \gets  (T^{\pi_k})^m v_k +\epsilon_k}$ instead of the update in the first equation in \eqref{eq: bad alg value upadte}. The value $\delta_k=0$ corresponds to having access to the exact model. Generally, one could leverage the model for a complete immediate solution instead of using Algorithm~\ref{alg:hmPI}, but here we consider cases where this cannot be done due to, e.g., too large of a state-space. In this case, we can approximately estimate the value and use a multiple-step greedy operator with access to the exact model. Indeed, this setup is conceptually similar to that taken in AlphaGoZero  \cite{silver2017mastering}. 
Figure~\ref{fig: approximate res} exhibits the results. The heatmap values are $||v^*-v^{\pi_f} ||_\infty,$ where $\pi_f$ is the algorithms' output policy after $4\cdot 10^6$ queries to the simulator. Both NC-$hm$-PI and $hm$-PI converge to a better value as $h$ increases. However, this effect is stronger in the latter compared to the former, especially for small values of $m$. This demonstrates how $hm$-PI is less sensitive to approximation error. This behavior corresponds to the $hm$-PI error bound in Theorem~\ref{adp}, which decreases as $h$ increases. 


\begin{figure}
\includegraphics[scale=0.33]{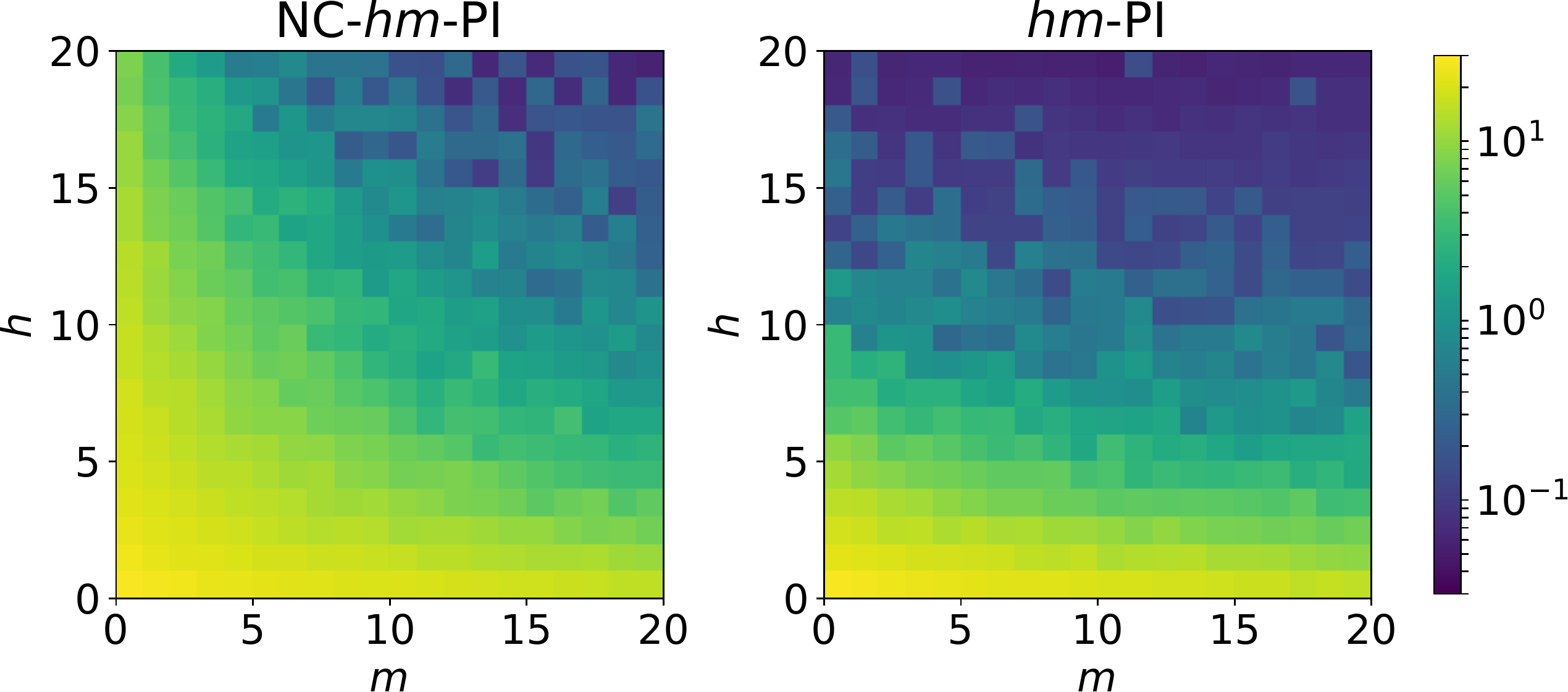}
\caption{Distance from optimum (lower is better) for NC-$hm$-PI and $hm$-PI in the presence of evaluation noise.
	The heatmap values are $||v^*-v^{\pi_f} ||_\infty,$ where $\pi_f$ is the algorithms' output policy after $4\cdot 10^6$ queries to the simulator.  The standard error of the results is given in Appendix \ref{supp: more results}.
}
\label{fig: approximate res}
\end{figure}
\section{Summary and Future Work}

In this work, we formulated, analyzed and tested two approaches for relaxing the evaluation stage of $h$-PI -- a multiple-step greedy PI scheme. The first approach backs up $v$ and the second backs up $T^{h-1}v$ or $T^{\pi_h}T^{h-1}v$ (see Remark~\ref{remark: root backup}). Although the first might seem like the natural choice, we showed it performs significantly worse than the second, especially  when combined with short-horizon evaluation, i.e., small $m$ or $\lambda$. Thus, due to the intimate relation between $h$-PI and state-of-the-art RL algorithms (e.g., \cite{silver2017mastering}), we believe the consequences of the presented results could lead to better algorithms in the future.

Although we established the non-contracting nature of the algorithms in Section \ref{sec: wrong hmPI hlambdaPI}, we did not prove they would necessarily not converge. We believe that further analysis of the non-contracting algorithms is intriguing, especially given their empirical converging behavior in the noiseless case (see Section~\ref{sec: experiments}, Figure \ref{fig: noisless res}). Understanding when the non-contracting algorithms perform well is of value, since their update rules are much simpler and easier to implement than the contracting ones. 

To summarize, this work highlights yet another difference between 1-step based and multiple-step based PI methods, on top of the ones presented in \cite{beyond2018efroni,efroni2018multiple}. Namely, multiple-step based methods introduce a new degree of freedom in algorithm design: the utilization of the planning byproducts. We believe that revealing additional such differences and quantifying their pros and cons is both intriguing and can have meaningful algorithmic consequences.



\bibliographystyle{aaai}\bibliography{hPI} 

\appendix

\onecolumn

\section{Proof of Lemma \ref{lemma: help value improvement}}\label{supp: help value improvement}
Since $(v,\pi)$ is $h$-greedy consistent we have that,
\begin{align*}
T^{h-1}v \leq T^{\pi}T^{h-1} v.
\end{align*}
By remembering that $(T^{\pi})^{l-1}$,for any $l \in  \mathbb{N} \setminus \{0\}$, is a monotonic operator we have that
\begin{align*}
(T^{\pi})^{l-1} T^{h-1}v\leq (T^{\pi})^l T^{h-1} v.
\end{align*}
We can concatenate the inequalities and conclude by observing that ${\lim_{l\rightarrow\infty}(T^{\pi})^l T^{h-1} v = v^\pi}$, since $T^{\pi}$ is a contraction operator with a fixed point $v^\pi$.

\section{Affinity of $T^\pi$ and Consequences}
In this section we prove, for completeness, an important property of $T^\pi$ (which was also described in \cite{beyond2018efroni}[Appendix~B]).
\begin{lemma} \label{lemma: left distributive}
Let $\{v_i,\lambda_i\}_{i=0}^\infty$ be a series of value functions, $v_i\in \mathbb{R}^{|\mathcal{S}|}$, and positive real numbers, $\lambda_i\in \mathbb{R}^+$, such that $\sum_{i=0}^\infty \lambda_i=1$. Let $T^\pi$ be a fixed policy Bellman operator and $n\in \mathbb{N}$. Then,
\begin{align*}
&T^\pi(\sum_{i=0}^\infty \lambda_i v_i )= \sum_{i=0}^\infty \lambda_i T^\pi v_i,\\
&(T^\pi)^n(\sum_{i=0}^\infty \lambda_i v_i )= \sum_{i=0}^\infty \lambda_i (T^\pi)^n v_i.
\end{align*}
\end{lemma}

\begin{proof}
Using simple algebra and the definition of $T^\pi$ (see Definition \ref{def: Tpi}) we have that
\begin{align*}
T^\pi(\sum_{i=0}^\infty \lambda_i v_i ) &= r^\pi +\gamma P^\pi (\sum_{i=0}^\infty \lambda_i v_i) =r^\pi +\sum_{i=0}^\infty \lambda_i \gamma P^\pi  v_i\\
&=\sum_{i=0}^\infty \lambda_i  \left(r^\pi + \gamma P^\pi v_i \right)=\sum_{i=0}^\infty \lambda_i  T^\pi v_i.
\end{align*}
The second claim is a result of the first claim and is proved by iteratively applying the first relation.
\end{proof}
\section{Proof of Proposition \ref{proposition: help value improvement2}}\label{supp: help value improvement2}

The proof goes as follows.
\begin{align}
v^*-(T^{\pi_h})^m T^{h-1}v &\leq v^*-T^{\pi_h}T^{h-1}v \label{eq: supp prop m lhs}\\
&=v^*-T^{h}v \nonumber\\
&=(T^{\pi_*})^h v^*-T^{h}v \nonumber\\
&\leq (T^{\pi_*})^h v^*-(T^{\pi_*})^h v \nonumber\\
&\leq \gamma ^h (P^{\pi_*})^h( v^*- v)\leq \gamma ^h || v^*- v||_\infty \nonumber.
\end{align}
The first relation holds due to Lemma \ref{lemma: help value improvement}, the second relation holds since $\pi_h\in \G_h(v)$, and the last relation holds since $(P^{\pi_*})^h$ is a stochastic matrix. To prove similar result for the second claim we merely change the first relation, to
\begin{align}
v^*-T^{\pi_h}_\lambda T^{h-1}v &= v^*-(1-\lambda)\sum_i \lambda^i(T^{\pi_h})^{i+1} T^{h-1}v \label{eq: supp prop lambda lhs}\\
&\leq v^*-(1-\lambda)\sum_i \lambda^i T^{\pi_h} T^{h-1}v \nonumber\\
&=v^*-T^{\pi_h}T^{h-1}v\nonumber,
\end{align} 
where the second relation holds according to Lemma \ref{lemma: help value improvement} since $(\pi_h,v)$ are $h$-greedy consistent.

Furthermore, 
\begin{align*}
(T^{\pi_h})^{m} T^{h-1}v \leq v^{\pi_h} \leq v^*,
\end{align*}
and
\begin{align*}
T^{\pi_h}_\lambda T^{h-1}v = &(1-\lambda)\sum_i \lambda^i(T^{\pi_h})^{i+1} (T^{\pi_h})^{h-1}v\\
\leq &(1-\lambda)\sum_i \lambda^i v^{\pi} = v^\pi \leq v^*,
\end{align*}
where the first inequality in both of the relations above holds due to Lemma \ref{lemma: help value improvement}, and the second inequality holds since ${v^\pi\leq v^*}$ for any $\pi$.

Thus, we have that ${v^*-(T^{\pi_h})^m T^{h-1}v,\ v^*-T^{\pi_h}_\lambda T^{h-1}v \geq 0}$, component-wise, and we can take the max-norm on the LHS of \eqref{eq: supp prop m lhs} and \eqref{eq: supp prop lambda lhs} to prove the statements.

\section{Proof of Theorem \ref{thm: contraction coefficient}}\label{supp: contraction coefficicent}
We begin with proving \eqref{eq: non-contraction m}.  We have that
\begin{align}
 v^* -(T^{\pi_h})^m v  =& v^* -v^{\pi_h}+v^{\pi_h}-(T^{\pi_h})^m v \nonumber\\
=& (T^{\pi_*})^h v^* -v^{\pi_h}+(T^{\pi_h})^m v^{\pi_h}-(T^{\pi_h})^m v \nonumber\\
=& (T^{\pi_*})^h v^* -v^{\pi_h}+\gamma^m(P^{\pi_h})^m \left(v^{\pi_h}- v\right) \nonumber\\
\leq& (T^{\pi_*})^h v^* -v^{\pi_h}+\gamma^m(P^{\pi_h})^m \left(v^*- v\right) \nonumber\\
\leq& (T^{\pi_*})^h v^* -T^h v+\gamma^m(P^{\pi_h})^m \left(v^*- v\right) \nonumber\\
\leq& (T^{\pi_*})^h v^* -(T^{\pi_*})^h v+\gamma^m(P^{\pi_h})^m \left(v^*- v\right) \nonumber\\
=& \gamma^h(P^{\pi_*})^h \left( v^* - v \right)+\gamma^m(P^{\pi_h})^m \left(v^*- v\right)  \nonumber\\
\leq& (\gamma^h+\gamma^m) \|v^*- v\|_\infty \label{eq: bad update contracts sometimes}.
\end{align}
The forth relation holds since $v^*\geq v^\pi$, the fifth relation holds due to Lemma \ref{lemma: help value improvement}, the sixth relation holds by the definition of the optimal Bellman operator (namely, $T^l v \geq (T^{\pi})^l v$ for any $v$ and $\pi$), and the last relation holds since $(P^{\pi_*})^h, (P^{\pi_h})^h$ are stochastic matrices.

We also have that
\begin{align}
(T^{\pi_h})^m v-v^* & = (T^{\pi_h})^m v-T^m v^* \nonumber\\
&= (T^{\pi_h})^m v-(T^{\pi_h})^m v^* \nonumber\\
&= \gamma^m(P^{\pi_h})^m( v- v^*)\nonumber\\
&\leq \gamma^m\| v- v^*\|_\infty \leq (\gamma^h+\gamma^m) \|v^*- v\|_\infty.  \label{eq: bad update contracts sometimes 2}
\end{align}
Where the first relation holds since $v^*$ is the fixed point of $T^m$, the second relation holds by the definition of the optimal Bellman operator, and the forth relation holds since $(P^{\pi_h})^m$ is a stochastic matrix.

Combining \eqref{eq: bad update contracts sometimes} and \eqref{eq: bad update contracts sometimes} yields
\begin{align}
|| v^* - (T^{\pi_h})^m v ||_\infty \leq & (\gamma^h +\gamma^m) || v^* - v ||_\infty.\label{eq: thm3 first statement}
\end{align}

The second statement is a consequence of the first statement.
\begin{align*}
|| v^* - T_\lambda^{\pi_h} v ||_\infty  = & || v^* - (1-\lambda)\sum_i \lambda^i (T^{\pi_h})^{i+1} v ||_\infty\\
 =& || (1-\lambda)\sum_i \lambda^i (v^*-(T^{\pi_h})^{i+1} v) ||_\infty\\
\leq& (1-\lambda)\sum_i \lambda^i || v^* - (T^{\pi_h})^{i+1} v ||_\infty\\
\leq &\left((1-\lambda)\sum_i \lambda^i (\gamma^{i+1} +\gamma^h) \right)  ||v^* - v ||_\infty \\
=& \left(\frac{\gamma(1-\lambda)}{1-\lambda\gamma} + \gamma^h \right) ||v^* - v ||_\infty.
\end{align*}
In the first relation we use the definition of $T^{\pi_h}_\lambda$, the third relation holds due to the triangle's inequality and the forth relation holds due to \eqref{eq: thm3 first statement}.

To conclude the proof we finish proving the tightness of \eqref{eq: non-contraction lambda} using the same construction given in the part of the proof that is in the paper's body:
	\begin{align*}
	(T^{\pi_h}_\lambda v)(s_0) &= \frac{1-\gamma^h}{1-\gamma}+\sum_{i=0}^\infty (\gamma\lambda)^i (\gamma(1-\gamma)v(s_{1}))\\
	&= \frac{1-\gamma^h}{1-\gamma}-\frac{\gamma(1-\lambda)}{(1-\gamma\lambda)}\cdot \frac{1}{1-\gamma}.
	\end{align*}
	
	See that
	\begin{align*}
	|(T^{\pi_h}_\lambda v)(s_0)-v^*(s_0)|=\left(\gamma^h+\frac{\gamma(1-\lambda)}{(1-\gamma\lambda)}\right)\frac{1}{1-\gamma}.
	\end{align*}
	Since $||(T^{\pi_h}_\lambda v)-v^*||_\infty = |(T^{\pi_h}_\lambda v)(s_0)-v^*(s_0)|$,
	\begin{align*}
	|| v^*-T^{\pi_h}_\lambda v ||_\infty &=\left(\gamma^h+\frac{\gamma(1-\lambda)}{(1-\gamma\lambda)}\right)  \frac{1}{1-\gamma} \\
	&= \left(\gamma^h+\frac{\gamma(1-\lambda)}{(1-\gamma\lambda)}\right)|| v^*-v ||_\infty
	\end{align*}

\section{$h$-Greedy Consistency in Each Iteration} \label{supp: invariant property noisy}

The following result is used to prove Theorem~\ref{thm: contraction coefficient}. According to it, the choice of $C_k$ leads to a sequence of $h$-greedy consistent policies and values in every iteration. 
\begin{lemma}\label{lemma: invariant property noisy}
	Let $\epsilon_k'=\epsilon_k-C_k e$, where  $C_k = \frac{\max \delta_{k+1}+ \gamma^{h-1}\max \epsilon_k -\gamma^h \min \epsilon_k}{\gamma^{h-1}(1-\gamma)}$ and $e$ is a vector of `ones' of dimension $|\mathcal{S}|.$ 
	For both $hm$-PI or $h\lambda$-PI, let the value function at the $k$-th iteration with the alternative error, $\epsilon'_k$, be $v'_k$. Let $\pi_{k+1}\in \approxhgreedy{\delta_{k+1}}{v'_k}.$  Then, in every iteration $k$ $(v'_k,\pi_{k+1})$ is $h$-greedy consistent; i.e.,
	\begin{align*}
	T^{h-1}v'_k \leq T^{\pi_{k+1}}T^{h-1} v'_k,
	\end{align*}
	and
	\begin{align*}
	v'_k-\epsilon'_k \leq  T^{\pi_{k+1}}T^{h-1} v'_k.
	\end{align*}
\end{lemma}
\begin{proof}[Proof of Lemma~\ref{lemma: invariant property noisy}: $hm$-PI part]
The proof goes by induction. The induction hypothesis is that $(\pi_k,v'_{k-1})$ is $h$-greedy consistent, $T^{h-1}v'_{k-1}\leq T^{\pi_k}T^{h-1} v'_{k-1}$, and we show it induces both of relations. The base case holds, i.e., $(\pi_1,v'_{0})$ is $h$-greedy consistent, due to $v_0'=v_{0}-d$, (see Remark \ref{remark: h greedy consistancy is easier}).

We start by proving that $(\pi_{k+1},v'_{k})$ for every $k$ by proving the induction step.
\begin{align}
T^{h-1}v'_k - T^{\pi_{k+1}}T^{h-1}v'_k& \leq T^{h-1}v'_k - T^{h}v'_k +\max \delta_{k+1}\nonumber\\
&=  T^{h-1}(v'_k-\epsilon_k') - T^{h}(v'_k-\epsilon_k')+\gamma^{h-1}\max \epsilon_k'-\gamma^h \min \epsilon_k'+\max \delta_{k+1}\nonumber\\
&= T^{h-1}(v'_k-\epsilon_k') - T^{h}(v'_k-\epsilon_k'), \label{eq: lemma nosiy improvement central}
\end{align}
where the last relation holds due to the choice of $\epsilon'_k$ and $C_k$, by which we get ${\gamma^{h-1}\max \epsilon_k'-\gamma^h \min \epsilon_k'+\max \delta_{k+1}=0}$.

We continue with the analysis from \eqref{eq: lemma nosiy improvement central},
\begin{align*}
T^{h-1}(v'_k-\epsilon_k') - T^{h}(v'_k-\epsilon_k')&= T^{h-1}(T^{\pi_k})^m T^{h-1}v'_{k-1} - T^{h}(T^{\pi_k})^m T^{h-1}v'_{k-1}\\
 &\leq  T^{h-1}T^{\pi_k}(T^{\pi_k})^m T^{h-1}v'_{k-1} - T^{h}(T^{\pi_k})^m T^{h-1}v'_{k-1} \\
  &\leq  T^{h-1}T(T^{\pi_k})^m T^{h-1}v'_{k-1} - T^{h}(T^{\pi_k})^m T^{h-1}v'_{k-1} =0.
\end{align*}

In the third relation we used Lemma \ref{lemma: help value improvement} due to the assumption that $(v'_{k-1},\pi_k)$ is $h$-greedy consistent and the monotonicity of $T^{h-1}$, in the forth relation we used the definition of the optimal Bellman operator, i.e., $ T^\pi \bar{v} \leq T\bar{v}$, and the monotonicity of $T^{h-1}$, and in the last relation we used $T^{h-1}T=T^h$ and recognized the two terms cancel one another.

This concludes that that for $hm$-PI the sequence of policies and alternative values are $h$-greedy consistent.

We now prove that ${v'_k-\epsilon_k' - T^{\pi_{k+1}}T^{h-1}v'_k\leq 0}$ for $hm$-PI.
\begin{align}
v'_k-\epsilon_k' - T^{\pi_{k+1}}T^{h-1}v'_k &\leq v'_k - T^{h}v_k +\max \delta_{k+1} \nonumber\\
&\leq v'_k-\epsilon_k' - T^{h}(v'_k-\epsilon_k') -\gamma^h \min \epsilon_k' +\max \delta_{k+1} \nonumber\\
&\leq  v'_k-\epsilon_k' - T^{h}(v'_k-\epsilon_k')\label{eq: lemma nosiy improvement central 2}
\end{align}

The last relation holds due to 
\begin{align*}
&-\gamma^h \min \epsilon_k' +\max \delta_{k+1}=-\gamma^h \min \epsilon_k +\max \delta_{k+1} -(1-\gamma^h)C_k\\
&=\max \epsilon_k(-\frac{1-\gamma^h}{1-\gamma})+\gamma^h \min \epsilon_k(\frac{1-\gamma^h}{\gamma^{h-1}(1-\gamma)}-1) +\max \delta_{k+1}(1-\frac{1-\gamma^h}{\gamma^{h-1}(1-\gamma)})\leq 0.
\end{align*}
See that the first and third terms are negative. Furthermore, $\min \epsilon_k \leq 0$ (if not, we can omit it in all previous analysis) and its coefficient is positive, the second term is also negative as well, and thus the entire expression is negative.

We continue with the analysis from \eqref{eq: lemma nosiy improvement central 2},
\begin{align*}
v'_k-\epsilon_k' - T^{\pi_{k+1}}T^{h-1}v'_k &\leq  v'_k-\epsilon_k' - T^{h}((v'_k-\epsilon_k'))\\
&\leq (T^{\pi_k})^m T^{h-1}v'_{k-1} - T^{h}(T^{\pi_k})^m T^{h-1}v'_{k-1}\\
&\leq (T^{\pi_k})^h(T^{\pi_k})^m T^{h-1}v'_{k-1} - T^{h}(T^{\pi_k})^m T^{h-1}v'_{k-1}\\
&\leq T^h(T^{\pi_k})^m T^{h-1}v'_{k-1} - T^{h}(T^{\pi_k})^m T^{h-1}v'_{k-1}=0.
\end{align*}

Where the third relation holds due to Lemma \ref{lemma: help value improvement}, and in the forth relation we used the definition of the optimal Bellman operator, i.e., $ (T^\pi)^h \bar{v} \leq T^h\bar{v}$.

Since $(\pi_k,v_{k-1})$ is $h$-greedy consistent due to the first claim we get
\begin{align*}
v'_k-\epsilon'_k - T^{\pi_{k+1}}T^{h-1}v'_k\leq 0.
\end{align*}

\end{proof}
To prove the statements for the $h\lambda$-PI we merely have to perform a minor change in \eqref{eq: lemma nosiy improvement central} and \eqref{eq: lemma nosiy improvement central 2} and to use the following Lemma, which is a consequence of Lemma \ref{lemma: help value improvement}.

\begin{lemma} \label{lemma: help lambda}
Let $\lambda\in[0,1], l\in \mathbb{N}$ and $(v,\pi)$ be $h$-greedy consistent. Then,
\begin{align*}
T^\pi_\lambda T^{h-1}v  \leq (T^{\pi})^l T^\pi_\lambda T^{h-1} v .
\end{align*}
\end{lemma}
\begin{proof}
We have that
\begin{align*}
T^\pi_\lambda T^{h-1}v &= (1-\lambda)\sum_i \lambda^i (T^\pi)^{i+1}T^{h-1}v\\
& \leq (1-\lambda)\sum_i \lambda^i (T^\pi)^{i+1+ l}T^{h-1}v \\
& = (1-\lambda)\sum_i \lambda^i (T^\pi)^{l}(T^\pi)^{i+1}T^{h-1}v \\
& = (T^\pi)^{l}\left((1-\lambda)\sum_i \lambda^i (T^\pi)^{i+1}T^{h-1}v \right)  = (T^\pi)^{l}T^\pi_\lambda T^{h-1}v.
\end{align*}
Where the third relation holds due to Lemma \ref{lemma: help value improvement},  and the forth relation holds by using Lemma \ref{lemma: left distributive}.

\end{proof}
\begin{proof}[Proof of Lemma~\ref{lemma: invariant property noisy}: $h\lambda$-PI part]
To prove that $h\lambda$-PI preserves the $h$-greedy consistency we start from \eqref{eq: lemma nosiy improvement central} and follow similar line of proof.
\begin{align*}
T^{h-1}(v'_k-\epsilon_k') - T^{h}(v'_k-\epsilon_k') &= T^{h-1}T^{\pi_k}_\lambda v'_{k-1} - T^{h}T^{\pi_k}_\lambda v'_{k-1}\\
&\leq  T^{h-1}T^{\pi_k}T^{\pi_k}_\lambda v'_{k-1} - T^{h}T^{\pi_k}_\lambda v'_{k-1}\\
&\leq  T^{h-1}T T^{\pi_k}_\lambda v'_{k-1} - T^{h}T^{\pi_k}_\lambda v'_{k-1} = 0.
\end{align*} 
Where the third relation holds due to Lemma \ref{lemma: help lambda}, and in the forth relation we used the definition of the optimal Bellman operator, i.e., $ T^\pi \bar{v} \leq T\bar{v}$, and the monotonicity of $T^{h-1}$.

This proves that the $h$-greedy consistency is preserved in $h\lambda$-PI as well. To prove the second statement for $h\lambda$-PI we start from \eqref{eq: lemma nosiy improvement central 2}.
\begin{align*}
v'_k-\epsilon_k' - T^{h}(v'_k-\epsilon_k') &=T^{\pi_k}_\lambda v'_{k-1}- T^{h}T^{\pi_k}_\lambda v'_{k-1}\\
&\leq (T^{\pi_k})^h T^{\pi_k}_\lambda v'_{k-1}- T^{h}T^{\pi_k}_\lambda v'_{k-1}\\
&\leq  T ^h T^{\pi_k}_\lambda v'_{k-1}- T^{h}T^{\pi_k}_\lambda v'_{k-1}=0.
\end{align*}

Where the third relation holds due to Lemma \ref{lemma: help lambda} and the monotonicity of $T^{h-1}$, and in the forth relation we used the definition of the optimal Bellman operator, i.e,, $ (T^\pi)^h \bar{v} \leq T^h\bar{v}$.
\end{proof}

\section{A Note on the Alternative $\lambda$-Return Operator}\label{supp: root backup}

In Remark \ref{remark: root backup} we defined an alternative $\lambda$-return operator, $\bar{T}_\lambda^\pi \eqdef (1-\lambda) \sum_{j=0}^\infty \lambda^j (T^\pi)^{j} v$.  We give here an equivalence form of this operator.
\begin{proposition}
For any $\pi$ and $\lambda\in [0,1]$
\begin{align*}
\bar{T}_\lambda^\pi v = v + \lambda(I-\gamma\lambda P^\pi)^{-1}(T^\pi v - v)
\end{align*}
\end{proposition}
\begin{proof}
This relation can be easily derived by using the equivalence in \eqref{def: T lambda pi}. We have that
\begin{align*}
\bar{T}_\lambda^\pi &\eqdef (1-\lambda) \sum_{j=0}^\infty \lambda^j (T^\pi)^{j} v\\
&=(1-\lambda)v + (1-\lambda) \sum_{j=1}^\infty \lambda^j (T^\pi)^{j} v\\
&=(1-\lambda)v + \lambda (1-\lambda) \sum_{j=0}^\infty \lambda^j (T^\pi)^{j+1} v\\
&=(1-\lambda)v +  \underset{\lambda T_\lambda^\pi v}{\underbrace{\lambda(1-\lambda) \sum_{j=0}^\infty \lambda^j (T^\pi)^{j+1} v}} =v + \lambda(I-\gamma\lambda P^\pi)^{-1}(T^\pi v - v),
\end{align*}
where in the last relation we used the equivalent form of $T^\pi_\lambda$ provided in \eqref{def: T lambda pi}.
\end{proof}

\section{More Experimental Results}\label{supp: more results}

In this section we add more empirical result on the convergence of the tested algorithms in Section \ref{sec: experiments} in the approximate case (as described in Section \ref{sec: experiments}). Specifically, we plot $|| v^*-v_k ||_\infty$ versus the total number of queries to the simulator, where $v_k$ is the value function. This complements the plot in Section \ref{sec: experiments}, there we plot $|| v^*-v^{\pi_f} ||_\infty$, where $v^{\pi_f}$ is the exact value of the policy that the algorithms output.

In the presence of errors, the value does not converge to a point in the, but only to a region. According to Theorem \ref{adp}, as $h$ increases, $hm$-PI is expected to converge to a `better' policy (i.e., closer to the optimal policy). As the results in Figure \ref{fig: appendix res} demonstrate, also the value function, $v$, of $hm$-PI converges to a better region than NC-$hm$-PI. This would be expected since a better policy would correspond to a better value function estimate. Furthermore, it is also observed that $hm$-PI converges faster than NC-$hm$-PI. This is again expected due to the possible non-contracting nature of this algorithm.

Lastly, in Figure \ref{fig: appendix res error bar}  the standard error, which corresponds to the mean results in Figure \ref{fig: approximate res}, is given.


{
\begin{figure*}[ht]
\noindent\makebox[\textwidth][c]{
\includegraphics[scale=0.35]{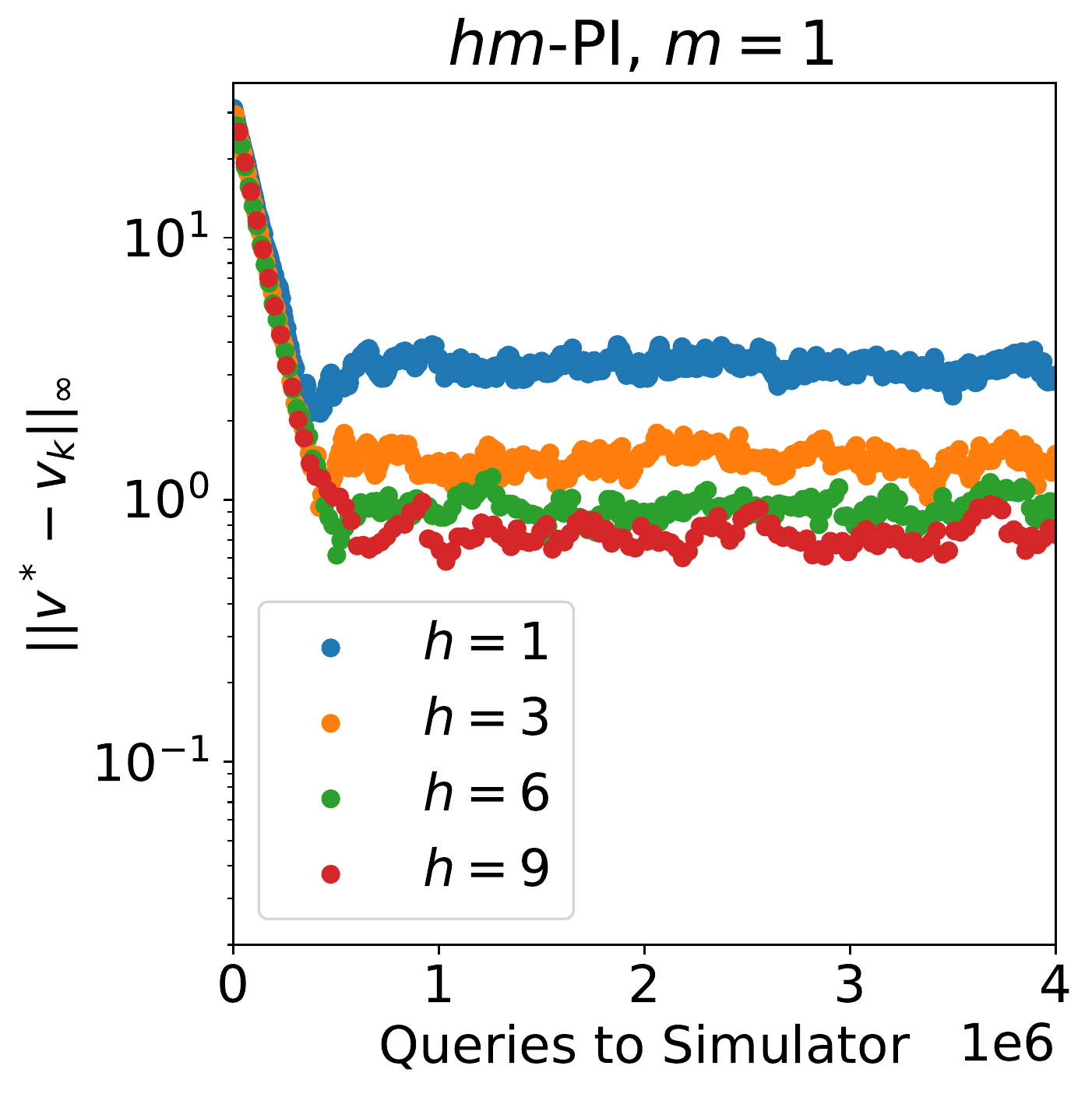}
\includegraphics[scale=0.35]{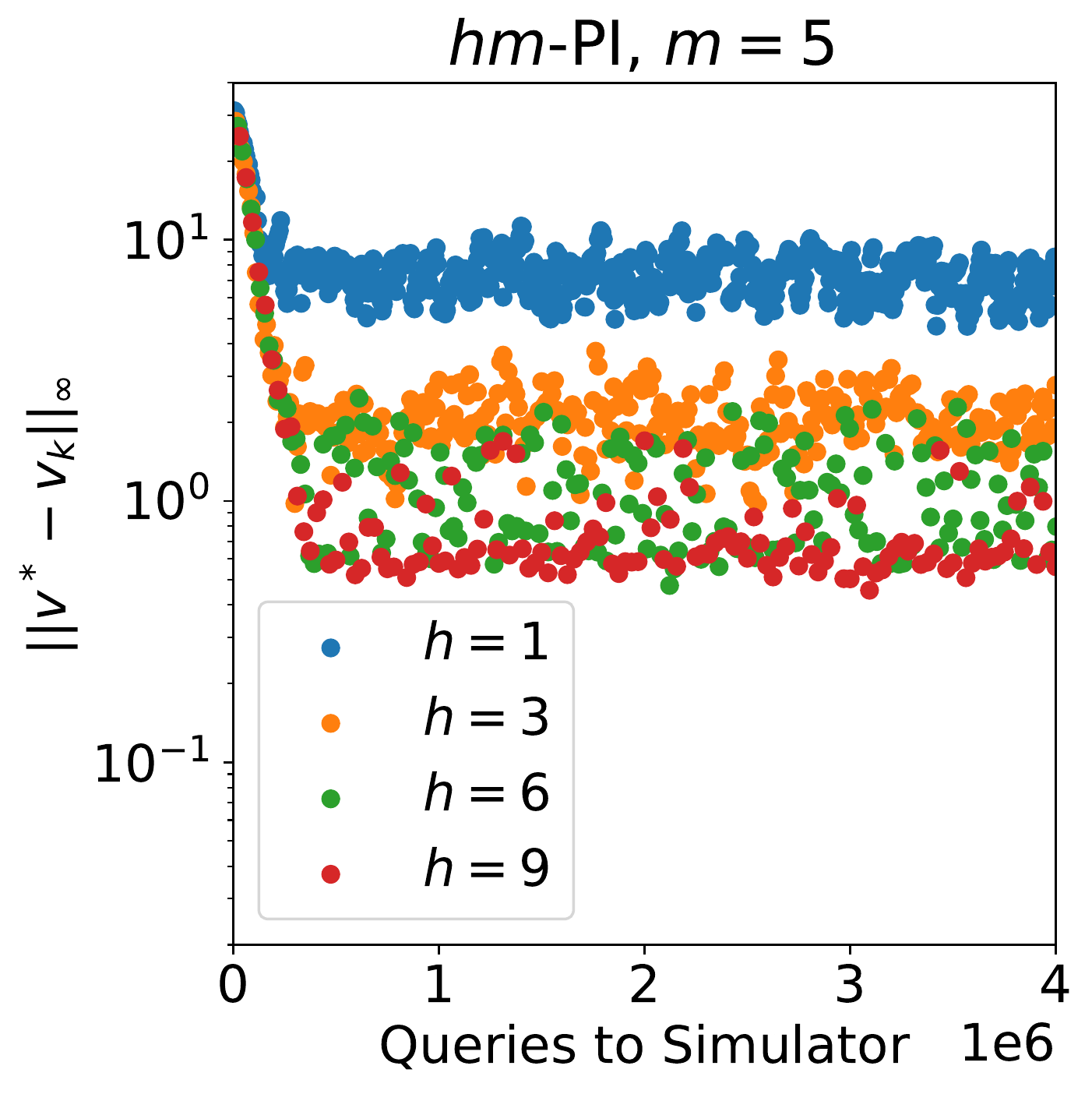}
\includegraphics[scale=0.35]{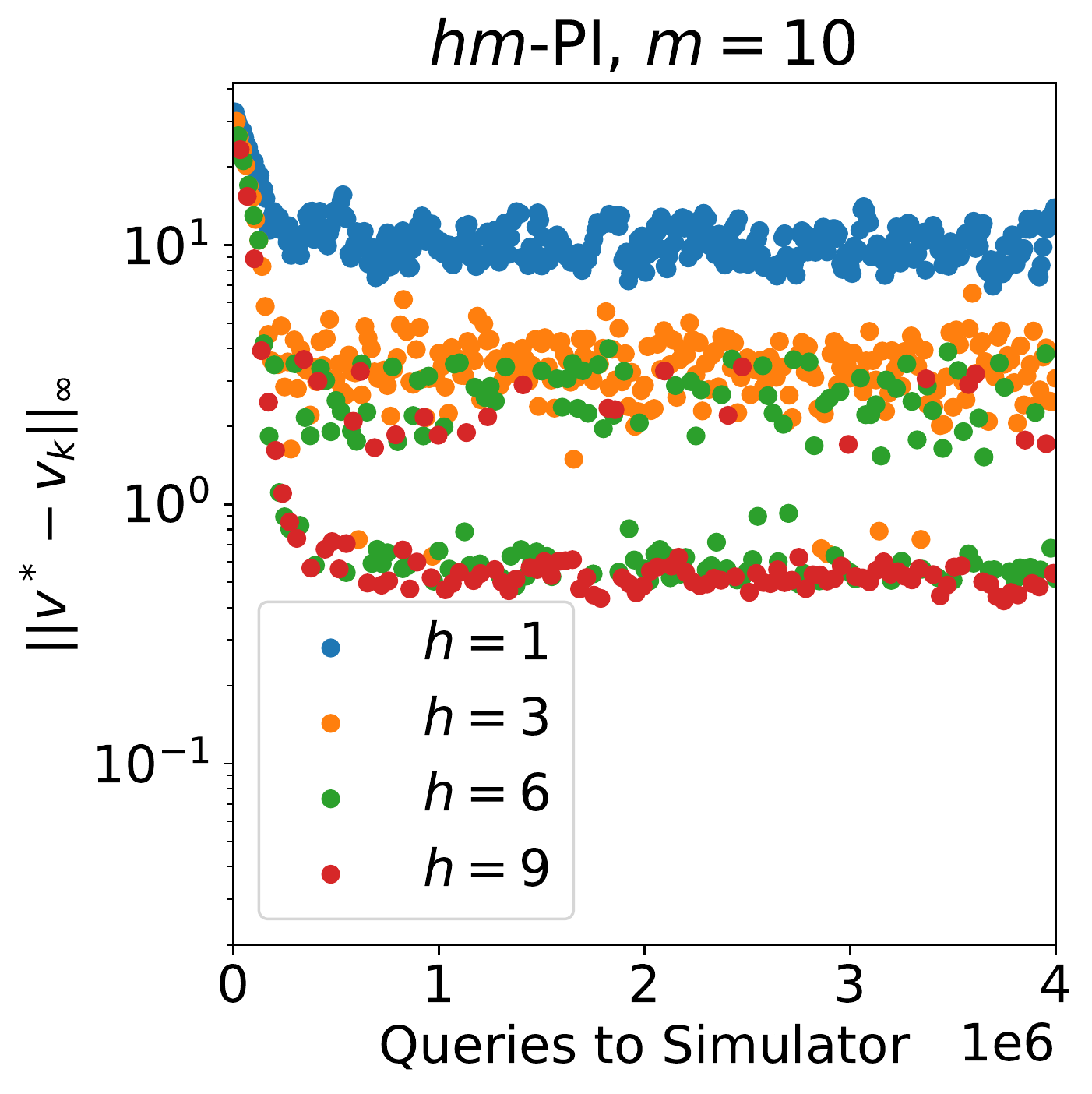}
}
\noindent\makebox[\textwidth][c]{
\includegraphics[scale=0.35]{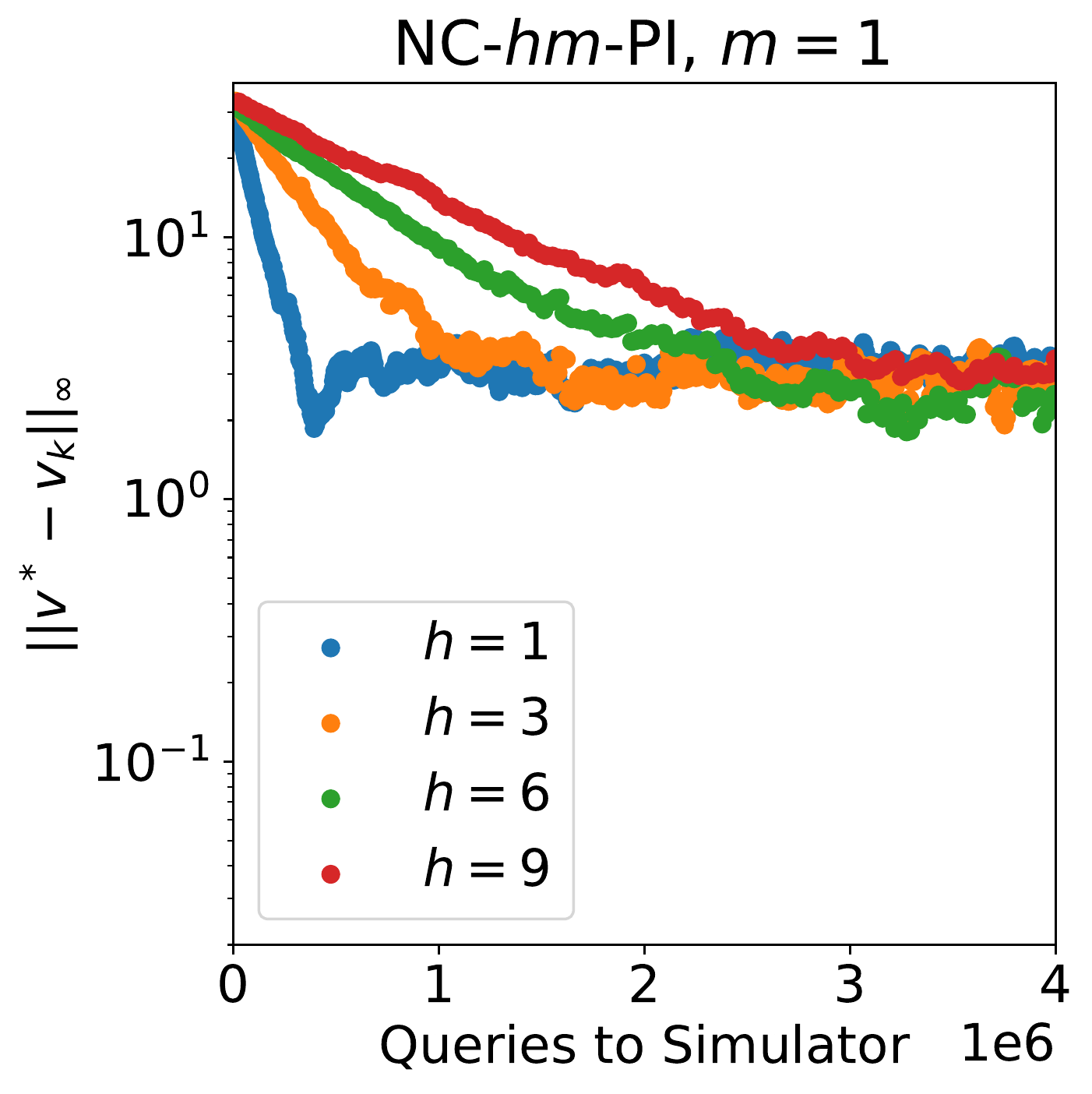}
\includegraphics[scale=0.35]{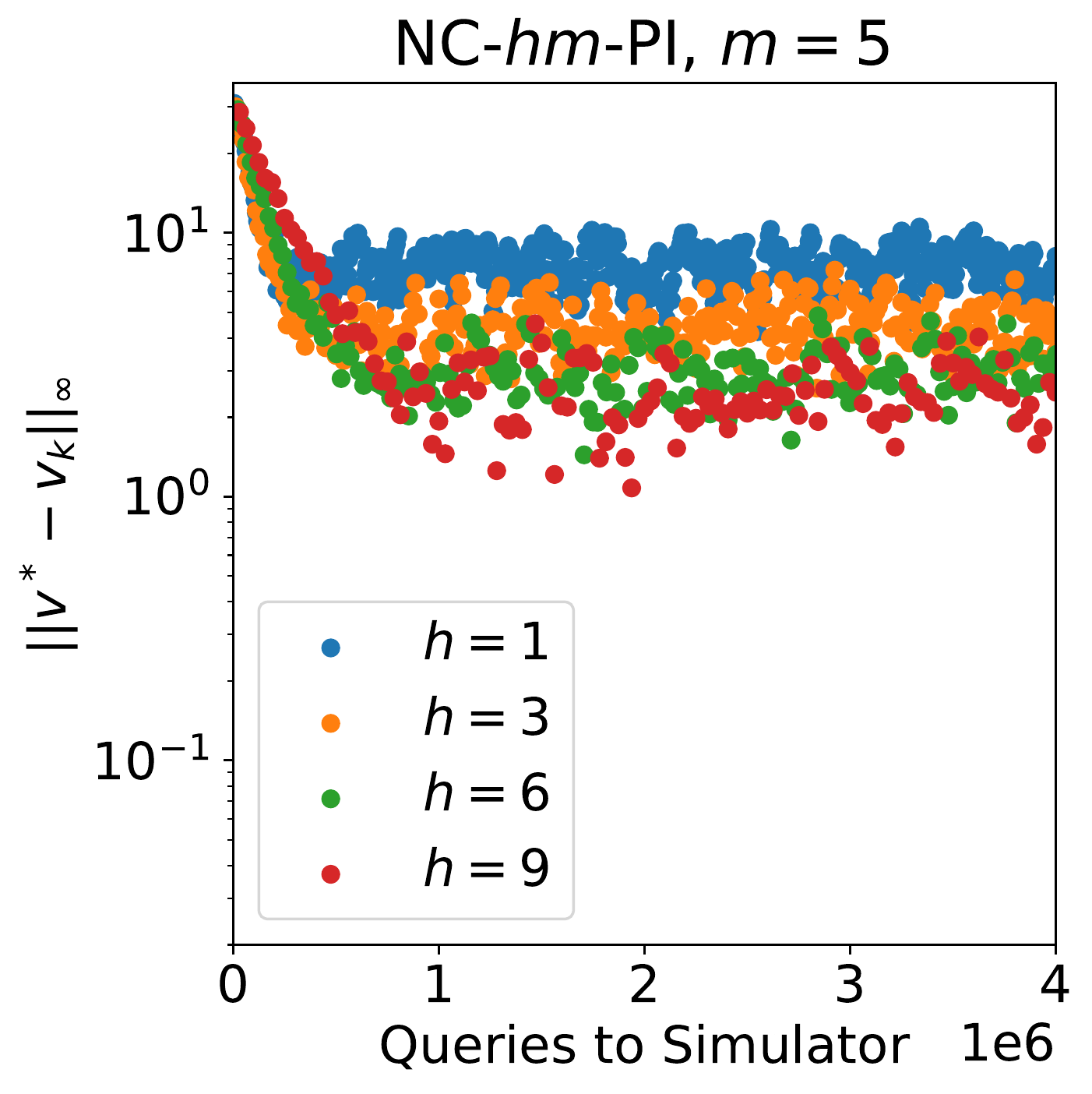}
\includegraphics[scale=0.35]{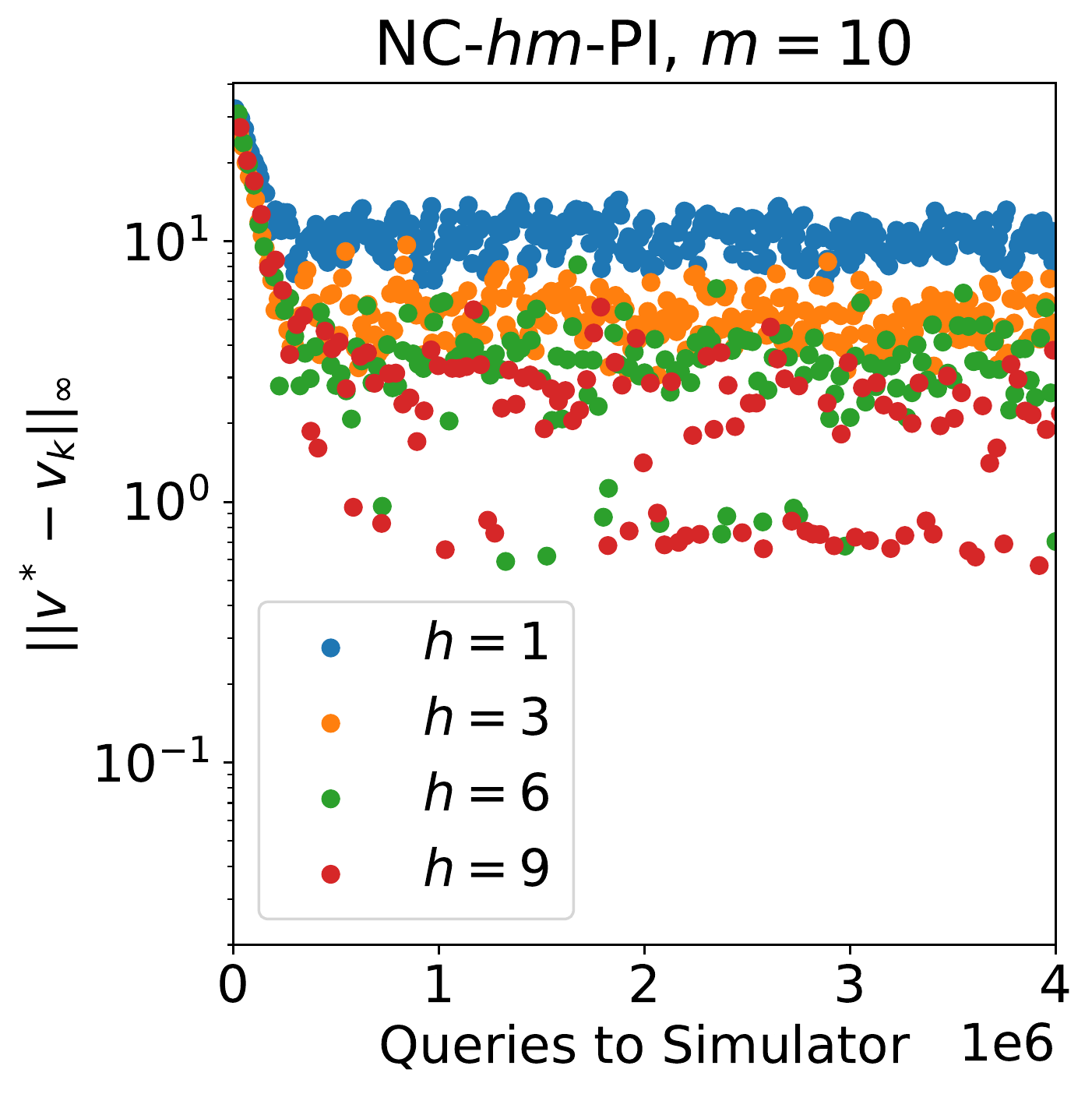}
}
\caption{$hm$-PI and NC-$hm$-PI performance for several $h$ and $m$ values. We measure $|| v^*-v_k ||_\infty$  versus total queries to simulator in each run. In this experiment we used $\forall s\in\mathcal{S},\ k,\ \epsilon_k(s)\sim U(-0.3,0.3),\ \delta_k(s)=0$, as described in Section~\ref{sec: experiments}.}
\label{fig: appendix res}
\end{figure*}}

\centering
\begin{figure}[ht]
\centering
\includegraphics[scale=0.4]{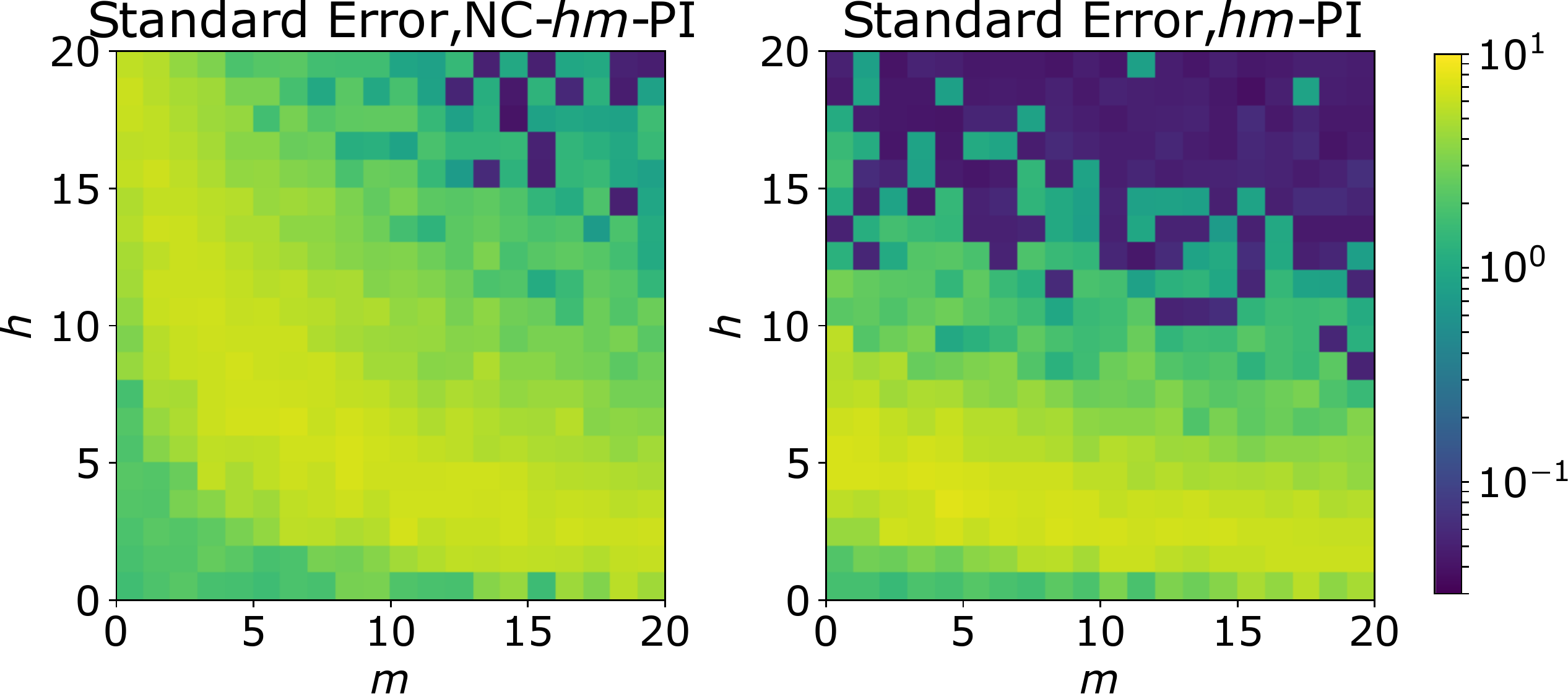}
\caption{Standard error versus $h$ and $m$ for the corresponding mean results given in Figure \ref{fig: approximate res}.
}
\label{fig: appendix res error bar}
\end{figure}

\end{document}